\documentclass{article}

\usepackage{microtype}
\usepackage{subcaption}
\usepackage{graphicx}
\usepackage{booktabs} % for professional tables

\usepackage{hyperref}

\usepackage[accepted]{icml2024}

\usepackage{amsmath}
\usepackage{amssymb}
\usepackage{mathtools}
\usepackage{amsthm}
\usepackage{pifont}

\usepackage[capitalize,noabbrev]{cleveref}

\theoremstyle{plain}
\newtheorem{theorem}{Theorem}[section]

\newtheorem{lemma}[theorem]{Lemma}

\theoremstyle{definition}
\newtheorem{definition}[theorem]{Definition}

\theoremstyle{remark}

\usepackage[textsize=tiny]{todonotes}

\usepackage{multirow}
\usepackage{comment}
\usepackage[algo2e]{algorithm2e} 
\usepackage{makecell}

\newcommand{\bI}{{\bf I}}
\newcommand{\bII}{{\bf II}}
\newcommand{\twopartdef}[4]
{ %\setstretch{1.5}
        \left\{
                \begin{array}{ll}
                        #1 & \mbox{if } #2 \\
                        #3 & \mbox{if } #4
                \end{array}
        \right.
}

\icmltitlerunning{Log Neural Controlled Differential Equations}

\begin{document}

\twocolumn[
\icmltitle{Log Neural Controlled Differential Equations: \\ The Lie Brackets Make a Difference}



\icmlsetsymbol{equal}{*}

\begin{icmlauthorlist}
\icmlauthor{Benjamin Walker}{Ox}
\icmlauthor{Andrew D. McLeod}{Ox}
\icmlauthor{Tiexin Qin}{HK}
\icmlauthor{Yichuan Cheng}{HK}
\icmlauthor{Haoliang Li}{HK}
\icmlauthor{Terry Lyons}{Ox}
\end{icmlauthorlist}

\icmlaffiliation{Ox}{Mathematical Institute, University of Oxford, UK}
\icmlaffiliation{HK}{Department of Electrical Engineering, City University of Hong Kong}

\icmlcorrespondingauthor{Benjamin Walker}{mlbenjamin walker@gmail.com}

%\icmlkeywords{Machine Learning, ICML}

\vskip 0.3in
]



\printAffiliationsAndNotice{}  % leave blank if no need to mention equal contribution

\begin{abstract}
%Abstracts must be a single paragraph, ideally between 4--6 sentences long.
The vector field of a controlled differential equation (CDE) describes the relationship between a \emph{control} path and the evolution of a \emph{solution} path. 
%The ODE is constructed from the iterated Lie brackets of the CDE's vector field and the truncated log-signature of the control. 
%At the end of each interval, the ODE's solution is a high-order approximation to the CDE's solution. 
Neural CDEs (NCDEs) treat time series data as observations from a control path, parameterise a CDE's vector field using a neural network, and use the solution path as a continuously evolving hidden state. 
As their formulation makes them robust to irregular sampling rates, NCDEs are a powerful approach for modelling real-world data.
% Given a partition of time, the Log-ODE method approximates the solution of a CDE. On each interval, the CDE is replaced by an autonomous ODE constructed using the Lie brackets of the CDE vector field and the log-signature of the control. 
Building on neural rough differential equations (NRDEs), we introduce Log-NCDEs, a novel, effective, and efficient method for training NCDEs. 
The core component of Log-NCDEs is the Log-ODE method, a tool from the study of rough paths for approximating a CDE's solution.
%To ensure the NCDE's vector field is regular enough for the Log-ODE method to be well defined, this paper proves a novel result bounding the $\mathrm{Lip}(\gamma)-$norm for a class of fully connected neural networks when $1<\gamma\leq 2$. 
%Furthermore, this paper details how the Lie brackets of a neural network can be efficiently calculated using existing machine learning tools. 
%On a range of multivariate time series datasets of length up to $50,000$, Log-NCDEs are shown to outperform NCDEs, NRDEs, and three state-of-the-art models; the linear recurrent unit, S5, and MAMBA.
Log-NCDEs are shown to outperform NCDEs, NRDEs, the linear recurrent unit, S5, and MAMBA on a range of multivariate time series datasets with up to $50{,}000$ observations.
\end{abstract}

\section{Introduction}

\subsection{Multivariate Time Series Modelling}

Neural controlled differential equations (NCDEs) offer a number of advantages for modelling real-world multivariate time series. These include being robust to irregular sampling rates and decoupling the number of forward passes through their neural network from the number of observations in the time series. However, there exists a gap in performance between NCDEs and current state-of-the-art approaches for time series modelling, such as S5, the linear recurrent unit (LRU), and MAMBA \citep{S5, orvieto2023resurrecting, gu2023mamba}. 

This paper demonstrates that, on a range of multivariate time series benchmarks, NCDEs can outperform current state-of-the-art approaches by utilising a tool from the study of rough paths, the Log-ODE method. We refer to this new approach as Log-NCDEs\footnote{\url{https://github.com/Benjamin-Walker/Log-Neural-CDEs}}.

\subsection{Neural Controlled Differential Equations}

Let $\{(t_i,x_i)\}_{i=0}^n$ denote a set of observations from a multivariate time series, where $x_i\in\mathbb{R}^{v-1}$. Let $X:[t_0,t_n]\rightarrow \mathbb{R}^{v}$ be a continuous interpolation, such that $X_{t_i}=(t_i, x_i)$, where subscript on time-dependent variables denotes evaluation. Let $h_t\in\mathbb{R}^u$ and $z_t\in\mathbb{R}^w$ be the NCDE's hidden state and output at time $t$ respectively. Let $\xi_{\phi}:\mathbb{R}^{v}\rightarrow\mathbb{R}^u$ and $f_{\theta}:\mathbb{R}^u\rightarrow\mathbb{R}^{u \times v}$ be neural networks, and $l_{\psi}:\mathbb{R}^u\rightarrow\mathbb{R}^{w}$ be a linear map, where $\phi$, $\theta$, and $\psi$ represent the learnable parameters. A NCDE is defined by
\begin{equation}
\begin{aligned}
\label{eq:ncde}
    h_{t_0} &= \xi_{\phi}(t_0, x_0), \\  h_t &= h_{t_0} + \int_{t_0}^t f_{\theta}(h_s)\text{d} X_s, \\ z_t&=l_{\psi}(h_t),
\end{aligned}
\end{equation}
for $t\in[t_0,t_n]$, where $f_{\theta}(h_s)\text{d} X_s$ is matrix-vector multiplication \citep{kidger2020neuralcde}. Details on the regularity required for existence and uniqueness of the solution to \eqref{eq:ncde} can be found in Appendix \ref{sec:app_ex_uniq}. By an extension of the Picard-Lindel\"of Theorem, a sufficient condition is $X$ being of bounded variation and $f_{\theta}$ being Lipschitz continuous \citep[Theorem 1.3]{lyons2007differential}.

NCDEs are an attractive option for modelling multivariate time series. They are universal approximators of continuous real-valued functions on time series data \citep[Theorem 3.9]{kidger2022neural}. Additionally, since they interact with time series data through $X$, NCDEs are unaware of when the time series was observed. This makes them robust to irregular sampling rates. Furthermore, the number of forward passes through $f_{\theta}$ when evaluating \eqref{eq:ncde} is controlled by the differential equation solver. This is opposed to recurrent models, where it is controlled by the number of observations $n$. By decoupling the number of forward passes through their neural network from the number of observations in the time series, NCDEs can mitigate exploding or vanishing gradients on highly-sampled time series. 

To make use of the techniques developed for training neural ordinary differential equations (ODEs) \citep{ChenRBD18}, NCDEs are typically rewritten as an ODE,
\begin{equation}
\label{eq:ncde_ode}
    \tilde{h}_t = \tilde{h}_{t_0} + \int_{t_0}^t g_{\theta, X}(\tilde{h}_s)\text{d}s,
\end{equation}
where $\tilde{h}_{t_0}=h_{t_0}$. \citet{kidger2020neuralcde} proposed taking $X$ to be a differentiable interpolation and
\begin{equation}
\label{eq:cde_diff}
    g_{\theta, X}(\tilde{h}_s) = f_{\theta}(\tilde{h}_s)\frac{\text{d} X}{\text{d} s}.
\end{equation}

\subsection{Neural Rough Differential Equations}

Neural rough differential equations (NRDEs) are based on the Log-ODE method, which is discussed briefly here, and in full detail in Section \ref{sec:math}. 

Given a set of intervals $\{[r_i, r_{i+1}]\}_{i=0}^{m-1}$ satisfying $t_0=r_0<\ldots< r_m=t_n$, the Log-ODE method replaces a CDE with an ODE on each interval. For a depth$-N$ method, the ODE on $[r_i,r_{i+1}]$ is defined by 
\begin{equation}
\label{eq:logode}
    g_{\theta,X}(\tilde{h}_s) = \bar{f}_{\theta}\left(\tilde{h}_s\right)\frac{\log(S^{N}(X)_{[r_i,r_{i+1}]})}{r_{i+1}-r_i},
\end{equation}
where $\bar{f}_{\theta}$ is constructed using the iterated Lie brackets of $f_{\theta}$ and $\log(S^{N}(X)_{[r_i,r_{i+1}]})$ is the depth-$N$ truncated log-signature of X over $[r_i,r_{i+1}]$ \cite{Boutaib2013}. Informally, $\bar{f}_{\theta}$ is a high order description of the vector field $f_{\theta}$ and $\log(S^{N}(X)_{[r_i,r_{i+1}]})$ is a high order description of the control path $X$ over $[r_i,r_{i+1}]$. Note that when using \eqref{eq:cde_diff}, $\tilde{h}_t$ is exactly $h_t$ for all $t\in[t_0,t_n]$, whereas when using \eqref{eq:logode}, $\tilde{h}_t$ is an approximation of $h_t$ when $t\in\{r_i\}_{i=1}^m$.

NRDEs replace \eqref{eq:cde_diff} with \eqref{eq:logode}, but instead of calculating $\bar{f}_{\theta}$ using the iterated Lie brackets of $f_{\theta}$, it is treated as a neural network $\bar{f}_{\theta}:\mathbb{R}^u\rightarrow\mathbb{R}^{u\times \beta(v,N)}$, where $\beta(v,N)$ is the dimension of a depth$-N$ truncated log-signature of a $v$ dimensional path. By neglecting the structure of $\bar{f}_{\theta}$, NRDEs are able to reduce the computational burden of evaluating the vector field, at the cost of increasing the output dimension of the neural network.

Compared to NCDEs, NRDEs can reduce the number of forward passes through the network while evaluating the model, as the vector field is autonomous on each interval $[r_i,r_{i+1}]$. This has been shown to lead to improved classification accuracy, alongside reduced time and memory-usage, on time series with up to 17,000 observations \citep{morrill2021neuralrough}. Furthermore, as it is no longer necessary to apply a differentiable interpolation to the time series data, NRDEs are applicable to a wider range of input signals. 
\subsection{Contributions}

%The Log-ODE method constructs its autonomous ODE using the truncated log-signature of the control and the iterated Lie brackets of the CDE's vector field. NRDEs use the truncated log-signature of the control, but neglect the iterated Lie brackets, as $\bar{f}_{\theta}$ is parameterised by a neural network. This reduces the computational cost of a forward pass through the network, at the cost of increasing the output dimension of the neural network. Furthermore, since NCDEs are universal approximators, the increase in dimension does not increase the model's expressivity.

This paper introduces Log-NCDEs, which build on NRDEs by constructing $\bar{f}_{\theta}$ using the iterated Lie brackets of a NCDE's vector field, $f_{\theta}$. For depth's $N\geq2$, this change drastically reduces the output dimension of the model's neural network, without impacting the model's expressivity. Furthermore, it improves model performance on every dataset considered in this paper. Calculating the Lie brackets requires that $f_{\theta}$ satisfies a regularity constraint, specifically being $\mathrm{Lip}(\gamma)$ for $\gamma\in(N-1,N]$. Section \ref{sec:lipgammaNN} presents a novel theoretical result bounding the $\mathrm{Lip}(\gamma)-$norm for a class of fully connected neural networks when $1<\gamma\leq 2$. Following this, Section \ref{sec:liebracketNN} details how to efficiently calculate the Lie brackets of a $\mathrm{Lip}(\gamma)$ neural network using standard machine learning tools. The paper concludes by showing that, over a range of multivariate time series benchmarks, Log-NCDEs outperform NCDEs, NRDEs, S5, LRU, and MAMBA.

\section{Mathematical Background}
\label{sec:math}

The following section provides a thorough mathematical description of the Log-ODE method. It will introduce $\mathrm{Lip}(\gamma)$ regularity, the Lie bracket of two vector fields, and the signature and log-signature of a path, alongside their respective spaces, the tensor algebra and the free Lie algebra.

\subsection{The Tensor Algebra}

\begin{definition}
    Let $U$, $V$, and $W$ be vector spaces. The tensor product space $U\otimes V$ is the unique (up to isomorphism) space such that for all bilinear functions $\kappa:U\times V \rightarrow W$ there exists a unique linear map $\tau:U\otimes V \rightarrow W$, such that $\kappa = \tau \circ \otimes$ \citep[Chapter 14]{roman2007advanced14}. 
\end{definition}

As an example, let $V=\mathbb{R}^2$ and $W=\mathbb{R}^3$. In this case, the tensor product is the outer product of the two vectors, and the resulting tensor product space is the space of $2\times 3$ matrices, $\mathbb{R}^2\otimes\mathbb{R}^3 = \mathbb{R}^{2\times 3}$. The tensor product of $v\in \mathbb{R}^2$ and $w\in \mathbb{R}^3$ is defined by
\begin{equation}
    v\otimes w = \begin{bmatrix} a \\ b \end{bmatrix} \otimes \begin{bmatrix} c \\ d \\ e \end{bmatrix} = \begin{bmatrix} ac & ad & ae \\ bc & bd & be \end{bmatrix},
\end{equation}
where any bilinear function $\kappa(v,w)$ can be written as a linear function $\tau(v\otimes w)$. 

\begin{definition}
     The tensor algebra space is the space
    \begin{equation}
        T((V)) = \{x=(x^0,x^1,\ldots) | x^i \in V^{\otimes i}\},
    \end{equation}
    where $V^{\otimes 0}=\mathbb{R}$, $V^{\otimes 1} = V$, and $V^{\otimes j} = V \otimes V^{\otimes j-1}$ \citep[Chapter 14]{roman2007advanced14}.
\end{definition}
Details on the choice of norm for $V^{\otimes i}$ when $V$ is a complete normed vector space, i.e. a Banach space, can be found in Appendix \ref{sec:app_norm}.

\subsection{$\mathrm{Lip}(\gamma)$ Functions}
\label{sec:lipgamma}

Let $V$ and $W$ be Banach spaces and $\mathbf{L}(V, W)$ denote the space of linear mappings from $V$ to $W$ equipped with the operator norm.
\begin{definition}
    A linear map $l\in\mathbf{L}(V^{\otimes j}, W)$ is $j$ symmetric if for all $v_1 \otimes \cdots \otimes v_j \in V^{\otimes j}$ and all bijective functions $p:\{1, \ldots, j\}\rightarrow \{1, \ldots, j\}$ \citep[Chapter 14]{roman2007advanced14},
    \begin{equation}
        l(v_1 \otimes \ldots \otimes v_j) = l(v_{p(1)} \otimes \cdots \otimes v_{p(j)}).
    \end{equation}
Let $\mathbf{L}_s(V^{\otimes j}, W)$ denote the set of all $j$ symmetric linear maps. 
\end{definition}
\begin{definition}
\label{def:lipgamma}
    Let $\gamma>0$, $k\in\mathbb{Z}$ such that $\gamma \in (k, k+1]$, $F$ be a closed subset of $V$, and $f^0:F\rightarrow W$. For $j\in\{1,\ldots,k\}$, let $f^j:F\rightarrow \mathbf{L}_s(V^{\otimes j}, W)$. The collection $(f^0, f^1, \ldots, f^k)$ is an element of $\text{Lip}(\gamma)$ if there exists $M\geq0$ such that for $j\in\{0,\ldots,k\}$,
    \begin{equation}
    \label{eq:lipbound1}
        \sup_{x\in F}||f^j(x)||_{\mathbf{L}(V^{\otimes j}, W)} \leq M,
    \end{equation}
    and for $j\in\{0,\ldots,k\}$, all $x,y \in F$, and each $v \in V^{\otimes j}$ \citep{EMStein},
    \begin{equation}
    \label{eq:lipbound2}
        \frac{\Big|\Big|f^j(y)(v)-\sum_{l=0}^{k-j}\frac{f^{j+l}(x)(v \otimes (y-x)^{\otimes l})}{l!}\Big|\Big|_{W}}{|x-y|_V^{\gamma-j}} \leq M.
    \end{equation}
\end{definition}
If $f=(f^0, f^1, \ldots, f^k)\in\text{Lip}(\gamma)$, then the $\text{Lip}(\gamma)-$norm, denoted $||f||_{\text{Lip}(\gamma)}$, is the smallest $M$ for which \eqref{eq:lipbound1} and \eqref{eq:lipbound2} hold. When $0<\gamma \leq 1$, then $k=0$ and $f^0\in\text{Lip}(\gamma)$ implies $f^0$ is bounded and $\gamma-$Hölder continuous. When $\gamma=1$, then $f^0$ is bounded and Lipschitz. In this paper, we are concerned with the regularity of vector fields on $\mathbb{R}^u$. In this case, $f\in\mathrm{Lip}(\gamma)$ for $\gamma \in (k, k+1]$ implies that $f$ is bounded, has $k$ bounded derivatives, and the $k^{\text{th}}$ derivative satisfies
\begin{equation}
    ||D^kf(y)-D^kf(x)|| \leq M|x-y|^{\gamma-k},
\end{equation}
for all $x,y\in\mathbb{R}^u$.

\subsection{The Free Lie Algebra}
\label{sec:liebracket}
\begin{definition}
    A Lie algebra is a vector space $V$ with a bilinear map $[\cdot,\cdot]:V\times V\rightarrow V$ satisfying $[w,w]=0$ and the Jacobi identity,
    \begin{equation}
        [[x,y],z]+[[y,z],x]+[[z,x],y] = 0,
    \end{equation}
    for all $w,x,y,z\in V$. The map $[\cdot,\cdot]$ is called the Lie bracket \citep{reutenauer1993free}.
\end{definition}
Any associative algebra, $(V,\times)$, has a Lie bracket structure with Lie bracket defined by 
\begin{equation}
    [x, y] = x\times y - y\times x,
\end{equation}
for all $x,y \in V$. For example, $V=\mathbb{R}^{n\times n}$ with the matrix product. For another example, consider the vector space of infinitely differentiable functions from $\mathbb{R}^u$ to $\mathbb{R}^u$ with pointwise addition, denoted $C^{\infty}(\mathbb{R}^u,\mathbb{R}^u)$. This space is a Lie algebra when equipped with the Lie bracket
\begin{equation}
    [a, b](x) = J_b(x)a(x) - J_a(x)b(x),
\end{equation}
for $a,b\in C^{\infty}(\mathbb{R}^u,\mathbb{R}^u)$ and $x\in\mathbb{R}^u$, where $J_a(x)\in\mathbb{R}^{u\times u}$ is the Jacobian of $a$ with entries given by $J_{a}^{ij}(x)=\partial_ja^i(x)$ for $i,j\in\{1,\ldots,u\}$ \citep{kirillov}.

\begin{definition}
\label{def:freeliealgebra}
    Let $A$ be a non-empty set, $L_0$ be a Lie algebra, and $\phi:A\rightarrow L_0$ be a map. The Lie algebra $L_0$ is said to be the free Lie algebra generated by $A$ if for any Lie algebra $L$ and any map $f:A\rightarrow L$, there exists a unique Lie algebra homomorphism $g:L_0\rightarrow L$ such that $g \circ \phi = f$ \citep{reutenauer1993free}.
\end{definition}
The free Lie algebra generated by $V$ is the space
\begin{equation}
    \mathfrak{L}((V)) = \{(l_0,l_1,\ldots):l_i\in L_i\},
\end{equation}
where $L_0=0$, $L_1=V$, and $L_{i+1}$ is the span of $[v,l]$ for $v\in V$ and $l\in L_i$.

\subsection{The Log-Signature}
\label{sec:logsig}

Let $X:[t_0,t_n]\rightarrow V$ have bounded variation and define 
\begin{equation}
    X^n_{[t_0,t_n]} = \underbrace{\int\cdots\int}_{\substack{u_1<\cdots<u_n \\ u_i\in [t_0,t_n]}}\text{d}X_{u_1}\otimes \cdots \otimes \text{d}X_{u_n} \in V^{\otimes n}.
\end{equation}
The signature of the path $X$ is 
\begin{equation}
    S(X)_{[t_0,t_n]}=(1,X^1_{[t_0,t_n]},\ldots,X^n_{[t_0,t_n]},\ldots) \in \tilde{T}((V)),
\end{equation}
where $\tilde{T}((V))=\{x\in T((V)) | x=(1,\ldots)\}$ \citep{lyons2007differential}. The depth-$N$ truncated signature is defined as $S^N(X)_{[t_0,t_n]}=(1,X^1_{[t_0,t_n]},\ldots,X^N_{[t_0,t_n]}) \in \tilde{T}^N(V)$. The signature is an infinite dimensional vector which describes the path $X$ over the interval $[t_0,t_n]$. In fact, assuming $X$ contains time as a channel, linear maps on $S(X)_{[t_0,t_n]}$ are universal approximators for continuous, real-valued functions of $X$ \citep{Lyons2014, Arribas2018DerivativesPU}. This property of the signature relies on the shuffle-product identity, which states that polynomials of the elements in a truncated signature can be rewritten as linear maps on the signature truncated at greater depth \citep{Ree1958LieEA, SigPrimer}. A consequence of the shuffle-product identity is that not every term in the signature provides new information about the path $X$. The transformation which removes the information redundancy is the logarithm.

\begin{definition}
    For $\mathbf{x}\in\tilde{T}((V))$, the logarithm is defined by 
    \begin{equation}
        \log(\mathbf{x}) = \log(1+\mathbf{t}) = \sum_{n=1}^\infty \frac{(-1)^n}{n}\mathbf{t}^{\otimes n},
    \end{equation}
    where $\mathbf{t}=(0,x^1,x^2,\ldots)$ \citep{reutenauer1993free}.
\end{definition}
It was shown by \citet{Chen1957IntegrationOP} that
\begin{equation}
    \log(S(X)_{[t_0,t_n]}) \in \mathfrak{L}((V)).
\end{equation}
This result plays a crucial role in the Log-ODE method. 

\subsection{The Log-ODE Method}
\label{sec:logode}

The vector field of a CDE is typically thought of as a matrix-valued function $f:\mathbb{R}^u\rightarrow\mathbb{R}^{u\times v}$. An equivalent formulation is $f$ being a linear map acting on $\text{d}X\in\mathbb{R}^{v}$ and returning a vector field on $\mathbb{R}^u$. In other words, $f(\cdot)\text{d}X:\mathbb{R}^u\rightarrow\mathbb{R}^u$. This formulation will prove more useful in the following section.

Let $X:[t_0,t_n]\rightarrow\mathbb{R}^v$ be a continuous path. The (truncated) log-signature is a map which takes $X$ to the (truncated) free Lie algebra $\mathfrak{L}^N(\mathbb{R}^v)$. If $f(\cdot)\text{d}X$ is restricted to smooth vector fields on $\mathbb{R}^u$, then $f(\cdot)\text{d}X$ is a linear map from $\mathbb{R}^v$ to the Lie algebra $C^{\infty}(\mathbb{R}^u, \mathbb{R}^u)$. By definition \ref{def:freeliealgebra}, there exists a unique linear map $\bar{f}$ from $\mathfrak{L}^N(\mathbb{R}^v)$ to the smooth vector fields on $\mathbb{R}^u$. Figure \ref{fig:logsig_triangle} is a schematic diagram of this relationship. Since $\bar{f}$ is a Lie algebra homomorphism, it can be defined recursively by 
\begin{equation}
\label{eq:logoderecurs1}
\bar{f}(\cdot)a = f(\cdot)a, \;\; a\in \mathbb{R}^v
\end{equation}
and 
\begin{equation}
\label{eq:logoderecurs2}
\bar{f}(\cdot)[a,b] = [\bar{f}(\cdot)a,\bar{f}(\cdot)b]
\end{equation}
for $[a,b]\in\mathfrak{L}^N(\mathbb{R}^v)$.
\begin{figure}
    \centering
    \includegraphics[width=\linewidth]{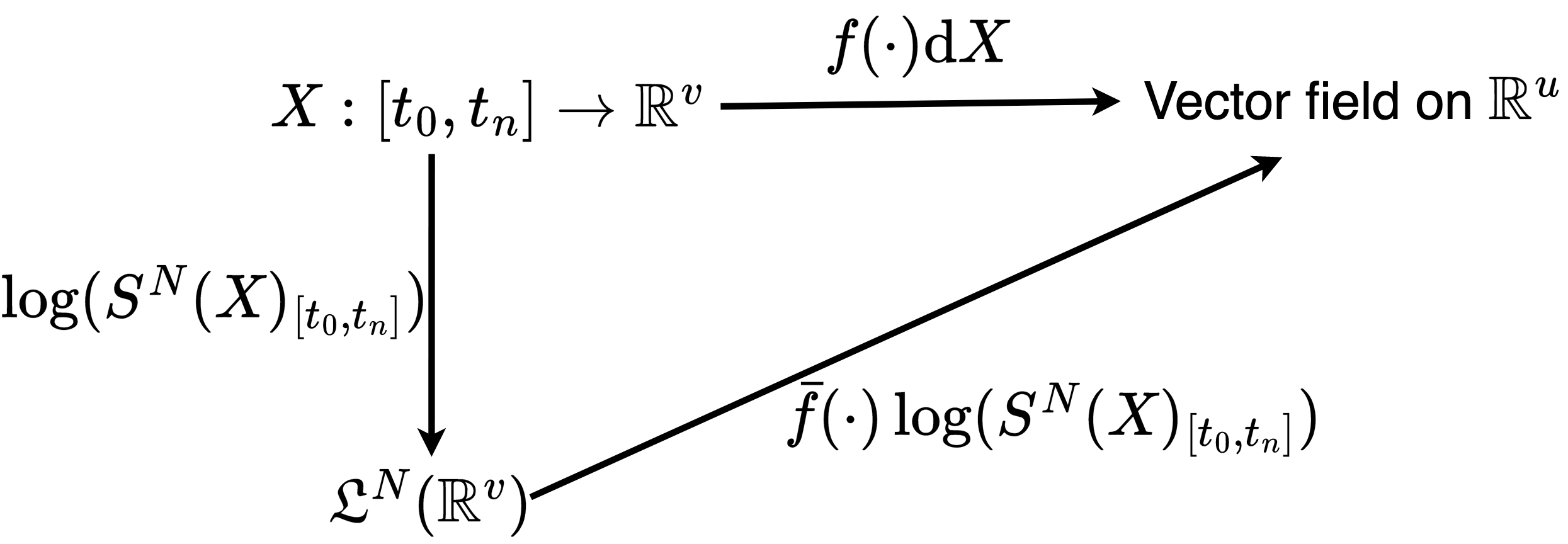}
    \caption{A schematic diagram of the Log-ODE method.}
    \label{fig:logsig_triangle}
\end{figure}

Over an interval, the Log-ODE method approximates a CDE using an autonomous ODE constructed by applying the linear map $\bar{f}$ to the truncated log-signature of the control, as seen in \eqref{eq:Log-NCDE} \citep{Lyons2014}. There exist theoretical results bounding the error in the Log-ODE method's approximation, including when the control and solution paths live in infinite dimensional Banach spaces \citep{Boutaib2013}. However, for a given set of intervals, the series of vector fields $\{g_{X}(\cdot)\}_{N=1}^{\infty}$ is not guaranteed to converge. In practice, $N$ is typically chosen as the smallest $N$ such that a reasonably sized set of intervals $\{r_i\}_{i=0}^m$ gives an approximation error of the desired level. A recent development has been the introduction of an algorithm which adaptively updates $N$ and $\{r_i\}_{i=0}^m$ \cite{bayer2023adaptive}. 

\section{Method}

\subsection{Log Neural Controlled Differential Equations}
\label{sec:logncde}

Log-NCDEs use the same underlying model as NRDEs
\begin{equation}
\label{eq:Log-NCDE}
\begin{aligned}
    h_t &= h_{t_0} + \int_{t_0}^t g_{\theta, X}(h_s)\text{d}s, \\ g_{\theta, X}(h_s) &= \bar{f}_{\theta}\left(h_s\right)\frac{\log(S^{N}(X)_{[r_i,r_{i+1}]})}{r_{i+1}-r_i},
\end{aligned}
\end{equation}
but with two major changes. First, instead of parameterising $\bar{f}_{\theta}$ using a neural network, it is constructed using the iterated Lie brackets of a NCDE's neural network, $f_{\theta}$, via \eqref{eq:logoderecurs1} and \eqref{eq:logoderecurs2}. Second, $f_{\theta}$ is ensured to be a $\mathrm{Lip}(\gamma)$ function for $\gamma\in(N-1,N]$. Figure \ref{fig:Log-NCDE} is a schematic diagram of a Log-NCDE.

\begin{figure*}
    \centering
    \includegraphics[width=0.68\linewidth]{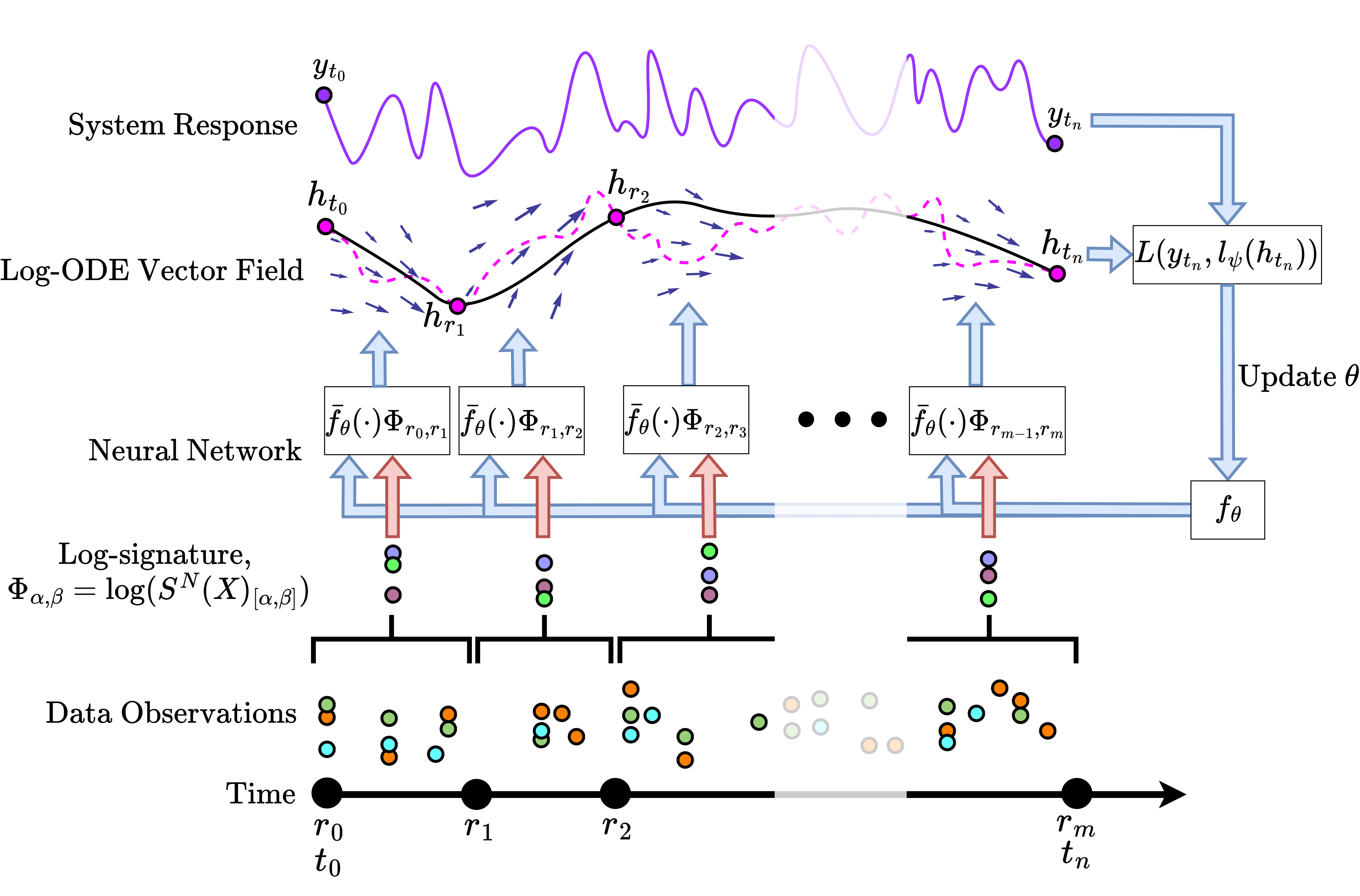}
    \caption{A schematic diagram of the training loop for a Log-NCDE. The coloured circles labelled data observations represent irregular samples from a time series. The purple line labelled system response is a potentially time varying label one wants to predict. The log-signatures of the data observations over each interval $[r_i,r_{i+1}]$ are combined with the iterated Lie brackets of $f_{\theta}$ to produce the vector field $g_{\theta, X}$ from \eqref{eq:Log-NCDE}. The pink dashed line represents the solution of \eqref{eq:ncde} and the solid black line represents the approximation obtained by solving \eqref{eq:Log-NCDE}. A linear map $l_{\psi}$ gives the Log-NCDE's prediction and a loss function $L(\cdot, \cdot)$ is used to update $f_{\theta}$'s parameters.}
    \label{fig:Log-NCDE}
\end{figure*}

\begin{figure}
    \centering
    \includegraphics[width=0.9\linewidth]{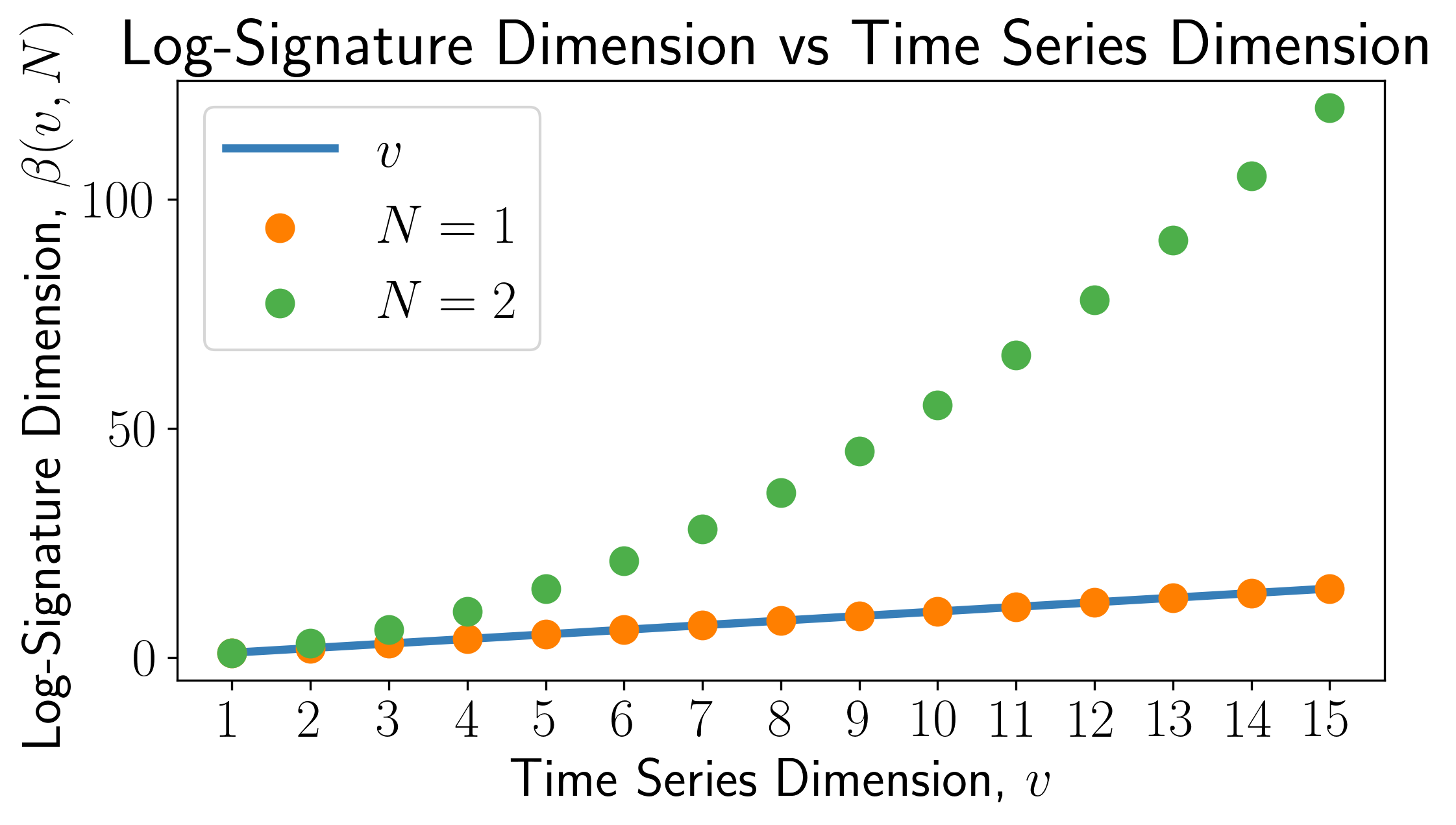}
    \caption{A plot of $\beta(v, N)$ against $v$ for $N=1,2$. The output dimension of a NRDE's neural network is $\mathbb{R}^{u\times \beta(v,N)}$, whereas for a Log-NCDE it is $\mathbb{R}^{u \times v}$.}
    \label{fig:logsig_dim}
\end{figure}

These changes have a major benefit. For $N>1$, Log-NCDEs are exploring a drastically smaller output space during training than NRDEs, while maintaining the same expressivity, as NCDEs are universal approximators. This is because the output dimension of $f_{\theta}$ is $u \times v$, whereas the output dimension of $\bar{f}_{\theta}$ is $u\times \beta(v,N)$. Figure \ref{fig:logsig_dim} compares these values for paths of dimension $v$ from $1$ to $15$ and truncation depths of $N=1$ and $N=2$. This benefit comes at the cost of needing to calculate the iterated Lie brackets when evaluating Log-NCDEs, which will be quantified in Section \ref{sec:cost} and explored empirically in Section \ref{sec:uea}. 

When $N=1$, \eqref{eq:Log-NCDE} simplifies to 
\begin{equation}
\label{eq:Log-NCDE_N1}
    g_{\theta, X}(h_s) = f_{\theta}\left(h_s\right)\frac{X_{r_{i+1}}-X_{r_i}}{r_{i+1}-r_i}.
\end{equation}
Hence, in this case the only difference between Log-NCDEs and NRDEs is the regularisation of $f_{\theta}$. Furthermore, \eqref{eq:Log-NCDE_N1} and \eqref{eq:cde_diff} are equivalent when $X$ is a linear interpolation. Therefore, the approach of NCDEs, NRDEs, and Log-NCDEs coincide when using a depth$-1$ Log-ODE approximation \citep{morrill2021neuralrough}.

\begin{comment}
\begin{algorithm}
\caption{NCDE vs. NRDE vs. Log-NCDE}
\label{alg:Log-NCDEs}
\textbf{Input:} NN and $\{t_i,x_i\}_{i=0}^n$ \\
\textbf{Output:} $\{t_i,h_i\}_{i=0}^n$ \\
\eIf{\emph{NCDE}}{
    control $\leftarrow \texttt{Interpolate}(\{t_i,x_i\}_{i=0}^n)$
    }{
    control $\leftarrow \texttt{Logsig}(\{t_i,x_i\}_{i=0}^n)$
    }
\eIf{\emph{Log-NCDE}}{
    vf $\leftarrow \texttt{Liebracket}(\text{NN})$
    }
    {
    vf $\leftarrow \text{NN}$
    }
    $\{t_i,h_i\}_{i=0}^n \leftarrow \texttt{CDESolve}(\text{vf}, \text{control}, \{t_i\}_{i=0}^n)$
\end{algorithm}
\end{comment}

\subsection{$\mathrm{Lip}(\gamma)$ Neural Networks}
\label{sec:lipgammaNN}
% There exist theoretical results linking the robutness of a learning algorithm to the algorithm's Lipschitz constant \citep{robustness}. Furthermore, there are results bounding the Lipschitz constant of a fully connected neural network (FCNN) \citep{NNLip}. Here, we extend these results to $\mathrm{Lip}(\gamma)$ for $1<\gamma\leq2$.

The composition of two $\mathrm{Lip}(\gamma)$ functions is $\mathrm{Lip}(\gamma)$ \citep[Lemma 2.2]{composition}. Hence, a simple approach to ensuring a fully connected neural network (FCNN) is $\mathrm{Lip}(\gamma)$ is to make each layer $\mathrm{Lip}(\gamma)$. This can be achieved by choosing an infinitely differentiable activation function, such as $\text{SiLU}$ \citep{ELFWING20183}. However, in practice this may not ensure sufficient regularity, as demonstrated by Theorem \ref{thm:lipnn}.

\begin{theorem}
\label{thm:lipnn}
    Let $f_{\theta}$ be a FCNN with input dimension $n_{in}$, hidden dimension $n_h$, depth $m$, and activation function $\text{SiLU}$. Assuming the input $\mathbf{x}=[x_1,\ldots,x_{n_{in}}]^T$ satisfies $|x_j|\leq1$ for $j=1,\ldots,n_{in}$, then $f_{\theta}\in\mathrm{Lip}(2)$ and 
    \begin{equation}
    \label{eq:lipnn}
    ||f_{\theta}||_{\text{Lip}(2)} \leq CP_{m!}\left(\{||W^i||_2, ||\mathbf{b}^i||_2\}_{i=1}^m\right)
    \end{equation}
    where $C$ is a constant depending on $n_{in}, n_h,$ and $m$, $\{W^i\}_{i=1}^m$ and $\{\mathbf{b}^i\}_{i=1}^m$ are the weights and biases of $i^{\text{th}}$ layer of $f_{\theta}$, and $P_{m!}$ is a polynomial of order $m!$.
\end{theorem}

\begin{proof}
    A proof is given in Appendix \ref{app:lipnn_proof}.
\end{proof}

Assuming that each layer $\{L^i\}_{i=1}^m$ of $f_{\theta}$ satisfies $||L^i||_{\text{Lip}(2)}=1$, an explicit evaluation of \eqref{eq:lipnn} gives
\begin{equation}
    ||f_{\theta}||_{\text{Lip}(2)} \leq 5^{2^{m-1}-1}.
\end{equation}
For a depth $7$ FCNN, this bound is greater than the maximum value of a single precision floating point number. Hence, it may be necessary to control $||f_{\theta}||_{\text{Lip}(2)}$ explicitly during training. This is achieved by modifying the neural network's loss function $L$ to
\begin{equation}
\label{eq:weightpenalty}
    L \mapsto L + \lambda\left(\sum_{i=1}^m||W^i||_2 + ||\mathbf{b}^i||_2\right),
\end{equation}
where $\lambda$ is a hyperparameter controlling the weight of the penalty. This is an example of weight regularisation, which has long been understood to improve generalisation in NNs \citep{weightreg, weightdecay}. Equation \ref{eq:weightpenalty} is specifically a variation of spectral norm regularisation \citep{Yoshida2017SpectralNR}. 

% In the experiments conducted for this paper, Log-NCDEs use a FCNN with $\text{SiLU}$ activation functions as their vector field $f_{\theta}$. Empirically, the regularisation proposed to control the $\mathrm{Lip}(2)$ norm is only beneficial on some datasets. Therefore, $\lambda$ is chosen using a hyperparameter optimisation with $0$ as an option.

\subsection{Constructing the Log-ODE Vector Field}
\label{sec:liebracketNN}

The linear map $\bar{f}_{\theta}$ in \eqref{eq:Log-NCDE} is defined recursively by \eqref{eq:logoderecurs1} and \eqref{eq:logoderecurs2}. Assuming $f_{\theta}(\cdot)a$ is infinitely differentiable, then $f_{\theta}(\cdot)a$ is an element of the Lie algebra $C^{\infty}(\mathbb{R}^u, \mathbb{R}^u)$ and 
\begin{equation}
    [f_{\theta}(\cdot)a,f_{\theta}(\cdot)b] = J_{f_{\theta}(\cdot)a}f_{\theta}(\cdot)b - J_{f_{\theta}(\cdot)b}f_{\theta}(\cdot)a,
\end{equation}
as discussed in Section \ref{sec:liebracket}. Let $\{e_j\}_{j=1}^{v}$ be the usual basis of $\mathbb{R}^{v}$. A choice of basis for $\mathfrak{L}^N(\mathbb{R}^{v})$ is a Hall basis, denoted $\{\hat{e}_k\}_{k=1}^{\beta(v, N)}$, which is a specific subset of up to the $(N-1)^{\text{th}}$ iterated Lie brackets of $\{e_j\}_{j=1}^{v}$ \citep{Hall1950ABF}. Rewriting \eqref{eq:Log-NCDE} using a Hall basis,
\begin{equation}
\label{eq:hall}
    \bar{f}_{\theta}\left(h_s\right)\frac{\log(S^{N}(X)_{[r_i,r_{i+1}]})}{r_{i+1}-r_i} =  \sum_{k=1}^{\beta(v, N)}\lambda_k \bar{f}_{\theta}(h_s)\hat{e}_k,
\end{equation}
where $\lambda_k$ is the term in the scaled log-signature corresponding to the basis element $\hat{e}_k$. Since each $\hat{e}_k$ can be written as iterated Lie brackets of $\{e_j\}_{j=1}^{v}$, it is possible to replace $\bar{f}_{\theta}(\cdot)\hat{e}_k$ with the iterated Lie brackets of $f_{\theta}(\cdot)e_i$ using \eqref{eq:logoderecurs1} and \eqref{eq:logoderecurs2}. Each $f_{\theta}(\cdot)e_i:\mathbb{R}^u \rightarrow \mathbb{R}^u$ is a vector field defined by the $i^{\text{th}}$ column of the neural network's output. Hence, $g_{\theta, X}$ can be evaluated at a point using iterated Jacobian-vector products (JVPs) of $f_{\theta}$.

\subsection{Computational Cost}
\label{sec:cost}

When the signature truncation depth $N$ is greater than $1$, NRDEs and Log-NCDEs incur an additional computational cost for each evaluation of the vector field, which we quantify here for $N=2$. Assume that a NCDE, NRDE, and Log-NCDE are all using an identical FCNN as their vector field, except for the dimension of the final layer in the NRDE. Let $m$ and $n_h$ be the depth and dimension of the FCNN's hidden layers respectively, and $u$ and $v$ be the dimension of $h_t$ and $X$ from \eqref{eq:ncde} respectively. Letting $F_{\text{x}}$ be the number of FLOPs to evaluate model x's vector field,
\begin{equation}
    \begin{aligned}
        F_{\text{NCDE}}&= 2un_h + 2(m-1)n_h^2 + 2uvn_h,\\
        F_{\text{NRDE}}&= 2un_h + 2(m-1)n_h^2 + u(v^2-v)n_h,\\
        F_{\text{Log-NCDE}}&= 3vF_{\text{NCDE}},
    \end{aligned}
\end{equation}
where the number of FLOPs to calculate a JVP is $3$ times that of evaluating the FCNN and $v$ JVPs of $f_{\theta}$ are needed to evaluate \eqref{eq:hall} when $N=2$ \citep[Chapter~4]{Griewank2000}. Log-NCDEs and NRDEs have the same asymptotic computational complexity in each variable. However, each JVP is evaluated at the same point $h_s$. This allows the computational burden of Log-NCDEs on high-dimensional time series to be reduced by constructing a batched function using Jax's \texttt{vmap} \citep{jax2018github}. The computational cost is evaluated empirically in Section \ref{sec:uea}.

\subsection{Limitations}
\label{sec:lim}
In this paper, we only consider Log-NCDEs which use a depth$-1$ or depth$-2$ Log-ODE approximation. This is due to the following two limitations. First, there are no theoretical results explicitly bounding the $\mathrm{Lip}(\gamma)$ norm of a neural network for $\gamma>2$. Second, the computational cost required to evaluate $g_{\theta, X}$ grows rapidly with the depth $N$. This can make $N>2$ computationally infeasible, especially for high-dimensional time series. Another general limitation of NCDEs is the need to solve the differential equation recursively, preventing parallelisation. This is in contrast to non-selective structured state-space models, whose underlying model is a differential equation that can be solved parallel in time \citep{S4}.

% Despite these limitations, Log-NCDEs will be shown achieve better performance than NCDEs and NRDEs while reducing the average run time. Furthermore, they will be shown to outperform current state-of-the-art deep learning approaches, including structured state-space models. 

\section{Experiments}

\subsection{Baseline Methods}

Log-NCDEs are compared against six models, which represent the state-of-the-art for a range of deep learning approaches to time series modelling. Four of these models are stacked recurrent models, which use the same general architecture, but with different recurrent layers. The architecture used in this paper is based on the one introduced by \citet{S5} and the four different recurrent layers considered are LRU, S5, MAMBA, and S6, the selective state-space recurrence introduced as a component of MAMBA \citep{orvieto2023resurrecting, S5, gu2023mamba}.
%The next baseline model is time series transformer (TST), a representation learning approach to multivariate time series representation learning, which has achieved state-of-the-art results on multiple datasets \cite{zerveas2021transformer}. 
The other two baseline models are continuous models; a NCDE using a Hermite cubic spline with backward differences as the interpolation and a NRDE \citep{kidger2020neuralcde, morrill2021neuralrough}. Further details on all model architectures can be found in Appendices \ref{app:baseline_models} and \ref{app:cde}.

\subsection{Toy Dataset}

We construct a toy dataset of $100{,}000$ time series with $6$ channels and $100$ regularly spaced samples each. For every time step, the change in each channel is sampled independently from the discrete probability distribution with density
\begin{equation}
    p(n) = \int_{n-0.5}^{n+0.5}\frac{1}{\sqrt{2\pi}}e^{-\frac{1}{2}x^2}\text{d}x,
\end{equation}
where $n\in\mathbb{Z}$. In other words, the change in a channel at each time step is a sample from a standard normal distribution rounded to the nearest integer. Figure \ref{fig:toy_example} is a plot of a sample path from the toy dataset. 

\begin{figure}
\centering
\includegraphics[width=0.8\linewidth]{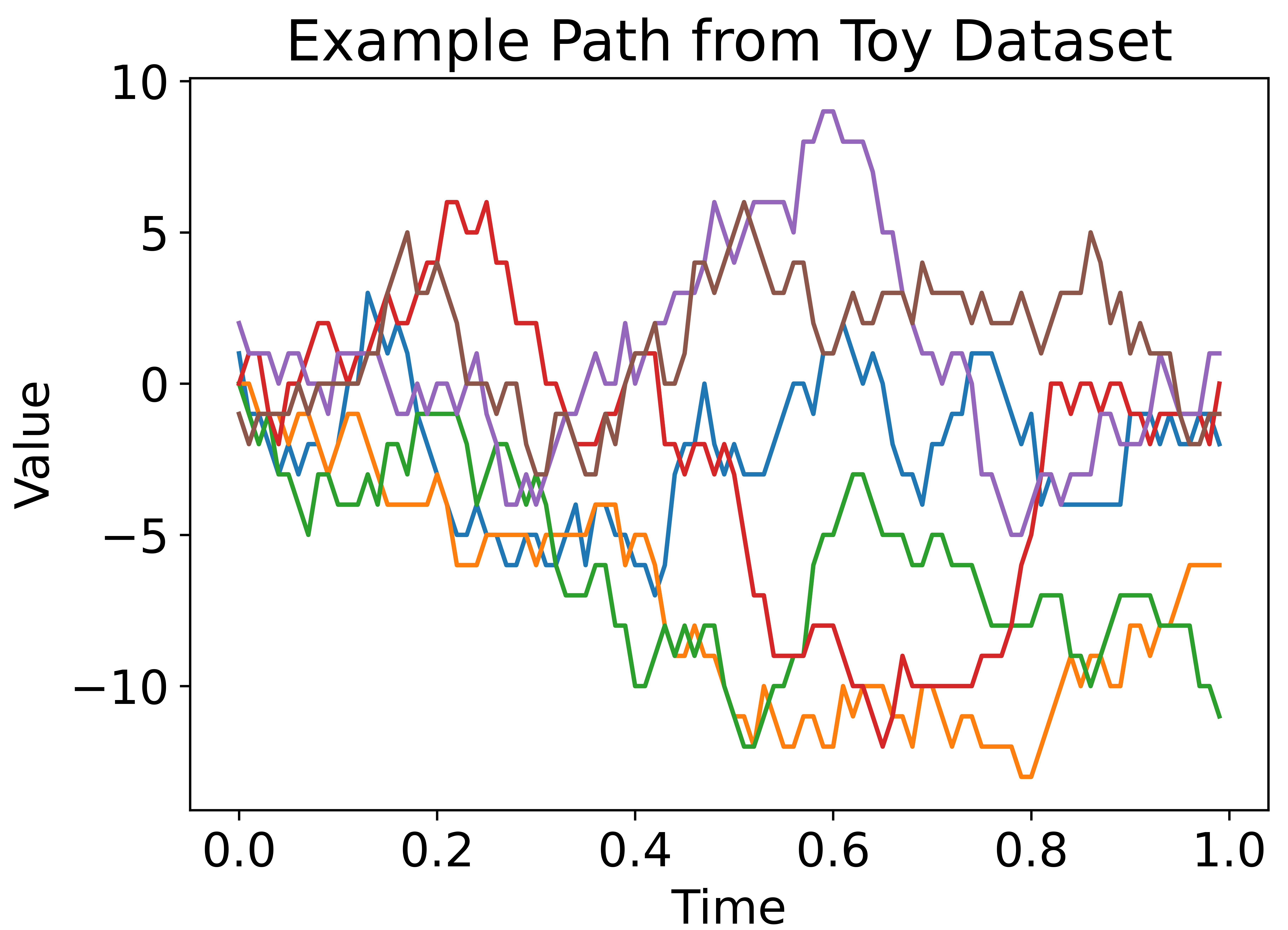}
\caption{An example path from the toy dataset, where each colour represents a channel in the path.}
\vspace{-0.25cm}
\label{fig:toy_example}
\end{figure}

We consider four different binary classifications on the toy dataset. Each classification is a specific term in the signature of the path which depends on a different number of channels.
\begin{enumerate}
    \item Was the change in the third channel, $\int_0^1\text{d}X^3_s$, greater than zero? 
    \item Was the area integral of the third and sixth channels, $\int_0^1\int_0^u\text{d}X^3_s\text{d}X^6_u$, greater than zero?
    \item Was the volume integral of the third, sixth, and first channels, $\int_0^1\int_0^v\int_0^u\text{d}X^3_s\text{d}X^6_u\text{d}X^1_v$, greater than zero?
    \item Was the $4$D volume integral of the third, sixth, first, and fourth channels, $\int_0^1\int_0^w\int_0^v\int_0^u\text{d}X^3_s\text{d}X^6_u\text{d}X^1_v\text{d}X^4_w$, greater than zero?
\end{enumerate} 

On this dataset, all models use a hidden state of dimension $64$ and Adam with a learning rate of $0.0003$ \citep{kingma2017adam}. LRU, S5, S6, and MAMBA use six blocks. NRDE and Log-NCDE take $r_{i+1}-r_i$ to be $4$ observations and the signature truncation depth $N$ to be $2$. Full hyperparameter choices are given in Appendix \ref{app:toy_details}.

\subsection{UEA Multivariate Time Series Classification Archive}

The models considered in this paper are evaluated on six datasets from the UEA multivariate time series classification archive (UEA-MTSCA)\footnote{As of June $1^{\text{st}}$ 2024, the EigenWorms dataset at \url{https://timeseriesclassification.com} has $23$ duplicated time series, which were removed for the experiments in this paper.}. These six datasets were chosen via the following two criteria. First, only datasets with more than $200$ total cases were considered. Second, the six datasets with the most observations were chosen, as datasets with many observations have previously proved challenging for deep learning approaches to time series modelling. Further details on the chosen datasets can be found in Appendix \ref{app:uea_ppg_details}. Following \citet{morrill2021neuralrough}, the original train and test cases are combined and resplit into new random train, validation, and test cases using a $70:15:15$ split. 

Hyperparameters for all models are found using a grid search over the validation accuracy on a fixed random split of the data. Full details on the hyperparameter grid search are in Appendix \ref{app:uea_ppg_details}. Having fixed their hyperparameters, models are compared on their average test set accuracy over five different random splits of the data. In order to compare models on their average GPU memory and run time, $1000$ steps of training are run on an NVIDIA RTX 4090. The average run time is estimated by combining the time for $1000$ training steps with the average total number of training steps from the five runs over the random data splits. 

\subsection{PPG-DaLiA}

PPG-DaLiA is a multivariate time series regression dataset, where the aim is to predict a person's heart rate using data collected from a wrist-worn device \citep{Reiss2019DeepPPG}. The dataset consists of fifteen individuals with around $150$ minutes of recording each at a maximum sampling rate of $128 \,\mathrm{Hz}$. There are six channels; blood volume pulse, electrodermal activity, body temperature, and three-axis acceleration. The data is split into a train, validation, and test set following a $70:15:15$ split for each individual. After splitting the data, a sliding window of length $49920$ and step size $4992$ is applied. 

Hyperparameters are found using the same method as for the UEA-MTSCA, but with validation mean squared error and slightly different hyperparameter choices given the high number of observations. Full details can be found in Appendix \ref{app:uea_ppg_details}. Having fixed their hyperparameters, models are compared on their average mean squared error over five different runs on the same fixed data split.

\section{Results}
\label{sec:results}

\subsection{Toy Dataset}

Figure \ref{fig:toy_results} compares the performance of the models considered in this paper on the four different toy dataset classifications. As expected, given that the classifications considered are solutions to CDEs, NCDE's are the best performing model. Since NRDEs and Log-NCDEs are fixed to $r_{i+1}-r_i$ being $4$ observations and $N=2$, they are both approximations of a CDE. Notably, Log-NCDEs consistently outperform NRDEs, providing empirical evidence that NRDEs do not always accurately learn the Lie bracket structure of $\bar{f}_{\theta}$. All of the stacked recurrent models perform well when the label depends on one or two channels. However, their performance begins to decrease for three channels, and only MAMBA performs well when the label depends on four channels.

\begin{figure*}
\centering
    \includegraphics[width=\linewidth]{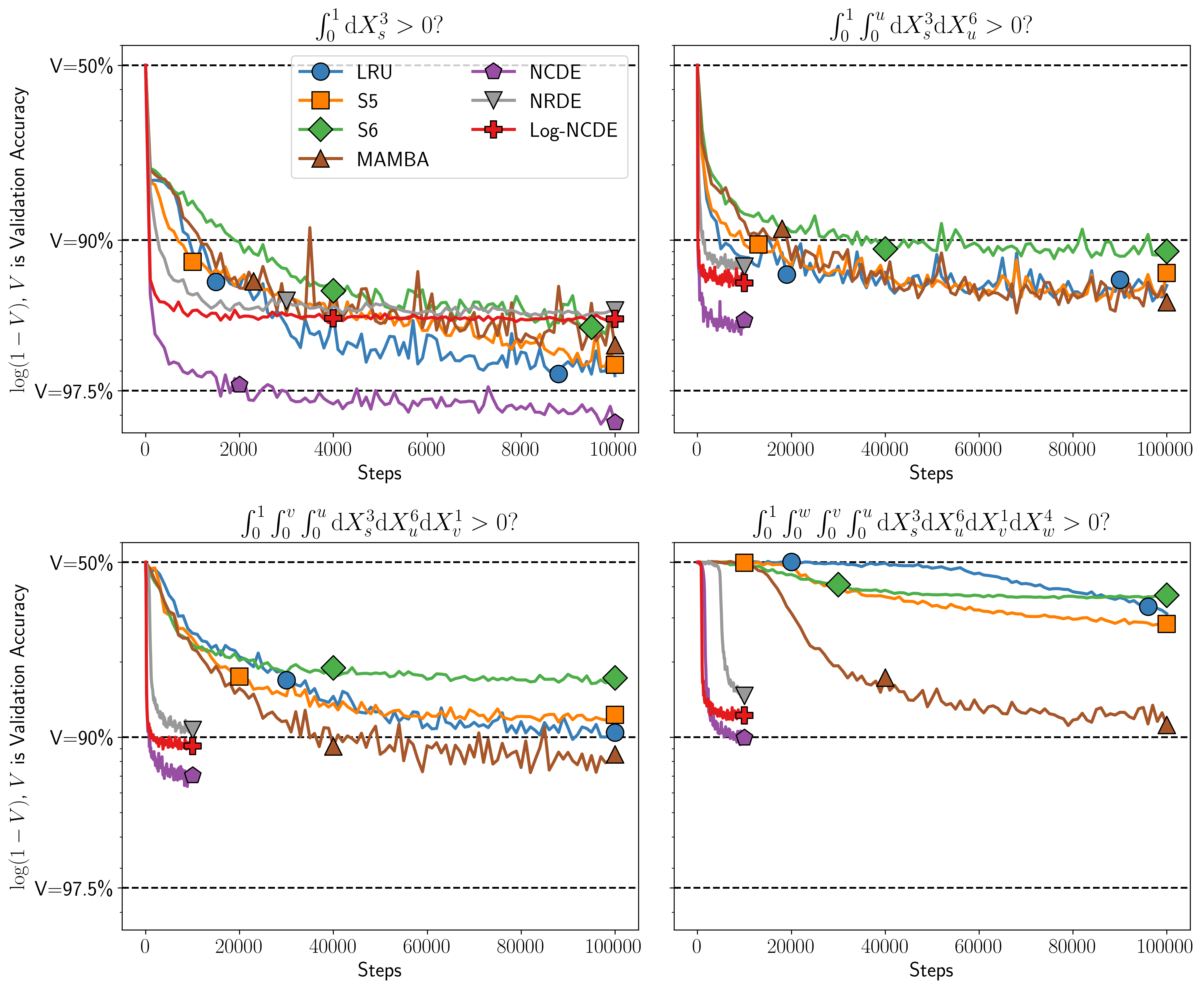}
    \caption{Validation accuracy against number of steps for LRU, S5, S6, MAMBA, NCDE, NRDE, and Log-NCDE on the four different classifications considered for the toy dataset.}
    \label{fig:toy_results}
\end{figure*}
\begingroup
\renewcommand{\arraystretch}{1.2}
\begin{table*}
\small
\caption{Test set accuracy on a subset of the UEA-MTSCA. The best performing model is highlighted in bold and the second best is underlined. The average accuracy and average rank are also reported.}
\label{tab:UEA_results_hypopt}
\centering
\begin{tabular}{l|c|c|c|c|c|c|c}
\hline
\multirow{2}{*}{Dataset} & \multicolumn{7}{c}{Method} \\ \cline{2-8}
& LRU & S5 & S6 & MAMBA & NCDE & NRDE & Log-NCDE \\ \hline
EigenWorms & $\mathbf{87.8 \pm 2.8}$ & $81.1 \pm 3.7$ & $85.0 \pm 16.1$ & $70.9 \pm 15.8$ & $75.0 \pm 3.9$ & $83.9 \pm 7.3$ & $\underline{85.6 \pm 5.1}$ \\
EthanolConcentration & $21.5 \pm 2.1$ & $24.1 \pm 4.3$ & $26.4 \pm 6.4$ & $27.9 \pm 4.5$ & $\underline{29.9 \pm 6.5}$ & $25.3 \pm 1.8$ & $\mathbf{34.4 \pm 6.4}$ \\
Heartbeat & $\mathbf{78.4 \pm 6.7}$ & $\underline{77.7 \pm 5.5}$ & $76.5 \pm 8.3$ & $76.2 \pm 3.8$ & $73.9 \pm 2.6$ & $72.9 \pm 4.8$ & $75.2 \pm 4.6$ \\
MotorImagery & $48.4 \pm 5.0$ & $47.7 \pm 5.5$ & $\underline{51.3 \pm 4.7}$ & $47.7 \pm 4.5$ & $49.5 \pm 2.8$ & $47.0 \pm 5.7$ & $\mathbf{53.7 \pm 5.3}$ \\
SelfRegulationSCP1 & $82.6 \pm 3.4$ & $\mathbf{89.9 \pm 4.6}$ & $82.8 \pm 2.7$ & $80.7 \pm 1.4$ & $79.8 \pm 5.6$ & $80.9 \pm 2.5$ & $\underline{83.1 \pm 2.8}$ \\
SelfRegulationSCP2 & $51.2 \pm 3.6$ & $50.5 \pm 2.6$ & $49.9 \pm 9.5$ & $48.2 \pm 3.9$ & $53.0 \pm 2.8$ & $\underline{\mathbf{53.7 \pm 6.9}}$ & $\underline{\mathbf{53.7 \pm 4.1}}$ \\ \hline
Av. & $61.7$ & $61.8$ & $\underline{62.0}$ & $58.6$ & $60.2$ & $60.6$ & $\mathbf{64.3}$ \\
Av. Rank & $\underline{3.5}$ & $4.0$ & $\underline{3.5}$ & $5.5$ & $4.5$ & $4.9$ & $\mathbf{2.1}$ \\
\end{tabular}
\end{table*}
\endgroup
\begingroup
\renewcommand{\arraystretch}{1.1}
\begin{table}
\small
\caption{Average GPU memory and run time for each model over the six datasets from the UEA-MTSCA experiments.}
%\vspace{0.2cm}
\label{tab:UEA_results_mem}
\centering
\begin{tabular}{l|c|c}
Model & Av. GPU Mem. (MB) & Av. run time (s) \\ \hline
LRU &  4121.67 & 466.09 \\
S5 &  2815.00 & 244.78 \\
S6 &  2608.00 & 578.15 \\
MAMBA &  4450.33 & 1553.83 \\
NCDE &  1759.67 & 6649.91 \\
NRDE &  2676.33 & 7284.20 \\
Log-NCDE &  1999.67 & 2128.32 \\
\end{tabular}
\end{table}
\endgroup

\subsection{UEA-MTSCA}
\label{sec:uea}

%Table \ref{tab:UEA_results_hypopt} reports the average and standard deviation of each model's test set accuracy over five different splits of the data. MAMBA achieves the lowest average test set accuracy. NCDEs and NRDEs achieve similar accuracies on all datasets except EigenWorms, the dataset with the most observations, where NRDEs significantly improve over NCDEs. However, NRDEs are still outperformed on average accuracy by three of the stacked recurrent models, LRU, S5, and S6. Log-NCDEs are the best performing model on average test set accuracy and average rank. Compared to NRDEs, they achieve an equal or higher average test set accuracy on all six datasets and a lower standard deviation on four datasets. These results highlight the improvement in performance which can be achieved by calculating the Lie brackets.

Table \ref{tab:UEA_results_hypopt} reports the average and standard deviation of each model's test set accuracy over five data splits. MAMBA achieves the lowest average accuracy. NCDEs and NRDEs have similar accuracies except on EigenWorms, the dataset with the most observations, where NRDEs significantly improve over NCDEs. However, NRDEs are still outperformed on average accuracy by three stacked recurrent models: LRU, S5, and S6. Log-NCDEs are the best performing model on average accuracy and rank. Compared to NRDEs, they achieve an equal or higher average accuracy on all six datasets and a lower standard deviation on four datasets. These results highlight the improvement in performance which can be achieved by calculating the Lie brackets.

Table \ref{tab:UEA_results_mem} details the average GPU memory and run time for each model, with the results for individual datasets given in Appendix \ref{app:mem_time}. The major contributors to NCDEs high average run time are time series with many observations. Using a depth$-2$ Log-ODE method decreases the computational burden on datasets with many observations for both NRDEs and Log-NCDEs. Although NRDEs and Log-NCDEs have the same asymptotic computational complexity, using a batched function to calculate the Lie brackets leads to Log-NCDEs having a lower computational burden on high-dimensional time series than NRDEs, as discussed in Section \ref{sec:cost}. Hence, Log-NCDEs have a lower average run time than both NCDEs and NRDEs. Despite these improvements, all four stacked recurrent models have lower average run times than Log-NCDEs. 

\subsection{PPG-DaLiA}

Table \ref{tab:PPG_results_hypopt} contains the average and standard deviation of each model's test set mean squared error on the PPG-DaLiA dataset. In contrast to the UEA-MTSCA experiments, MAMBA is the best performing stacked recurrent model. However, Log-NCDEs still achieve the best performance, obtaining the lowest average test set mean squared error and the second lowest standard deviation. 
%These results, in conjunction with those on the UEA-MTSCA datasets, demonstrate Log-NCDEs' applicability to diverse problem settings.

\begingroup
\renewcommand{\arraystretch}{1.1}
\begin{table}
\vspace{-0.4cm}
\small
\caption{Average test set mean squared error on the PPG-DaLiA dataset.}
\label{tab:PPG_results_hypopt}
%\vspace{0.2cm}
\centering
\begin{tabular}{c|c}
Model & MSE $(\times 10^{-2})$ \\ \hline
LRU & $12.17 \pm 0.49$ \\
S5 & $12.63 \pm 1.25$ \\
S6 & $12.88 \pm 2.05$ \\
MAMBA & $10.65 \pm 2.20$ \\
NCDE & $13.54 \pm 0.69$ \\
NRDE & $\underline{9.90\pm 0.97}$ \\
Log-NCDE & $\mathbf{9.56 \pm 0.59}$ 
\end{tabular}
\end{table}
\endgroup

\section{Discussion}

Recent theoretical work on the expressive power of structured state space models may provide an explanation for their performance on the toy dataset. It has been shown that the recurrent layer of non-selective state space models, such as S5, are unable to capture terms in the signature that depend on more than one channel. Instead, the computational burden is placed on the non-linear mixing in-between recurrent blocks. Furthermore, it has been shown that stacked selective state-space models, such as MAMBA, can capture higher order terms in the signature with only linear mixing layers \citep{cirone2024theoretical}. Although the toy dataset highlights a potential limitation of the stacked recurrent models, this did not translate into poor performance on the real-world datasets considered in this paper.

Log-NCDEs achieve the highest average accuracy on the UEA-MTSCA datasets and the lowest mean squared error on the PPG-DaLiA dataset. In particular, they improve upon NRDEs on every dataset. These results highlight the effectiveness of the Log-ODE method for improving the performance of NCDEs and the importance of calculating the Lie brackets when applying the Log-ODE method. Furthermore, despite increasing the computational cost of each vector field evaluation, Log-NCDEs have a lower average run time than both NCDEs and NRDEs on the UEA-MTSCA datasets. Empirical evidence suggests this is due to  the Log-ODE method improving efficiency on time series with many observations, and calculating the Lie brackets lowering the computational burden of the Log-ODE method on high-dimensional time series. In addition to achieving state-of-the-art performance on the regularly sampled datasets considered in this paper, Log-NCDEs maintain the ability of NCDEs to naturally handle irregular sampling, making them an attractive option for real-world applications. 

\section{Conclusion}

Building on NRDEs, this paper introduced Log-NCDEs, which utilise the Log-ODE method to train NCDEs in an effective and efficient manner. This required proving a novel theoretical result bounding the $\mathrm{Lip}(\gamma)$ norm of fully connected neural networks for $1<\gamma\leq 2$, as well as developing an efficient method for calculating the iterated Lie brackets of a neural network. A thorough empirical evaluation demonstrated the benefits of calculating the Lie brackets when applying the Log-ODE method. Furthermore, it showed that Log-NCDEs can achieve state-of-the-art performance on a range of multivariate time series datasets.

A reasonable direction of future work is extending Log-NCDE's to depth$-N$ Log-ODE methods for $N>2$. This would require proving an equivalent result to Theorem \ref{thm:lipnn} for $\gamma>2$. Furthermore, it would be necessary to address the computational cost of the iterated Lie brackets. This could be achieved by using a structured neural network with cheap Jacobian-vector products as the CDE vector field. Another avenue of future work could be incorporating the recently developed adaptive version of the Log-ODE method \citep{bayer2023adaptive}. 
%An alternative direction is improving the efficiency of Log-NCDE training. For example, as the vector field is $\mathrm{Lip}(\gamma)$, it is well approximated locally by a polynomial. Therefore, rather than evaluating the vector field and it's Lie brackets for each sample in a batch at each time step in the differential equation solve, information from nearby points could be leveraged to construct a polynomial approximation.

\section*{Impact Statement}

This paper presents Log Neural Controlled Differential Equations, a novel approach aimed at advancing the field of time series modelling. Potential applications of the method include healthcare, finance, and biology, where accurate time series modeling plays a crucial role. Despite the clear potential for positive impact, care must be taken to further understand the capabilities and limitations of the model before real-world deployment. Additionally, structured state-space models, an alternative approach to time series modelling, have recently been integrated into large language models (LLMs). The advancement of LLMs has many potential societal consequences, both positive and negative.

\section*{Acknowledgements}
Benjamin Walker was funded by the Hong Kong Innovation and Technology Commission (InnoHK Project CIMDA).
Andrew McLeod was funded in part by the EPSRC [grant number EP/S026347/1] and in part by The Alan Turing Institute under the EPSRC grant EP/N510129/1.
Terry Lyons was funded in part by the EPSRC [grant number EP/S026347/1], in part by The Alan Turing Institute under the EPSRC grant EP/N510129/1, the Data Centric Engineering Programme (under the Lloyd’s Register Foundation grant G0095), the Defence and Security Programme (funded by the UK Government) and the Office for National Statistics \& The Alan Turing Institute (strategic partnership) and in part by the Hong Kong Innovation and Technology Commission (InnoHK Project CIMDA).
The authors would like to acknowledge the use of the University of Oxford Advanced Research Computing (ARC) facility in carrying out this work. http://dx.doi.org/10.5281/zenodo.22558

%\newpage
\bibliography{Log-NCDE-ICML}
\bibliographystyle{icml2024}

%%%%%%%%%%%%%%%%%%%%%%%%%%%%%%%%%%%%%%%%%%%%%%%%%%%%%%%%%%%%%%%%%%%%%%%%%%%%%%%
%%%%%%%%%%%%%%%%%%%%%%%%%%%%%%%%%%%%%%%%%%%%%%%%%%%%%%%%%%%%%%%%%%%%%%%%%%%%%%%
% APPENDIX
%%%%%%%%%%%%%%%%%%%%%%%%%%%%%%%%%%%%%%%%%%%%%%%%%%%%%%%%%%%%%%%%%%%%%%%%%%%%%%%
%%%%%%%%%%%%%%%%%%%%%%%%%%%%%%%%%%%%%%%%%%%%%%%%%%%%%%%%%%%%%%%%%%%%%%%%%%%%%%%
\newpage
\appendix
\onecolumn
\section{Additional Mathematical Details}

\subsection{Existence and Uniqueness}
\label{sec:app_ex_uniq}

Let $V$ and $W$ be Banach spaces, $X:[0,T]\rightarrow V$ and $y:[0,T]\rightarrow W$ be continuous paths, and $f(\cdot)\nu$ be a linear map from $\nu\in V$ to vector fields on $W$. Assume that $X$, $Y$, and $f$ are regular enough for the integral 
\begin{equation}
    \int_{0}^tf(Y_s)\text{d}X_s
\end{equation}
to be defined for all $t\in [0,T]$ in the Young sense \cite{Young1936AnIO}. The path $Y$ is said to obey a controlled differential equation (CDE) if 
\begin{equation}
\label{eq:cde}
    Y_t = Y_{0} + \int_{0}^tf(Y_s)\text{d}X_s,
\end{equation}
for $t\in [0,T]$, where $Y_{0}\in W$ is the initial condition and $X$ is the control \cite{lyons2007differential}. The existence and uniqueness of the solution to a CDE depends on the smoothness of the control path $X$ and the vector field $f$. We will measure the smoothness of a path by the smallest $p\geq1$ for which the $p-$variation is finite and the smoothness of a vector field by the largest $\gamma>0$ such that the function is $\text{Lip}(\gamma)$ (defined in Section \ref{sec:lipgamma}).
\begin{definition} \emph{(Partition)}
    A partition of a real interval $[0,T]$ is a set of real numbers $\{r_i\}_{i=0}^m$ satisfying $0=r_0<\ldots<r_m=T$.
\end{definition}
\begin{definition} \emph{($p-$variation \cite{Young1936AnIO})}
    Let $V$ be a Banach space, $\mathcal{D}=(r_0, \cdots, r_m) \subset [0,T]$ be a partition of $[0,T]$, and $p\geq 1$ be a real number. The $p-$variation of a path $X:[0,T]\rightarrow V$ is defined as 
    \begin{equation}
        ||X||_{p} = \left[\sup_{\mathcal{D}}\sum_{r_i \in \mathcal{D}}|X_{r_i}-X_{r_{i+1}}|^p\right]^{\frac{1}{p}}.
    \end{equation} 
\end{definition}
\begin{theorem}
\label{th:existence}
    Let $1\leq p <2$ and $p-1 < \gamma \leq 1$. If $W$ is finite-dimensional, $X$ has finite $p-$variation, and $f$ is $\text{Lip}(\gamma)$, then (\ref{eq:cde}) admits a solution for every $y_{0} \in W$ \cite{Lyons1994DIFFERENTIALED}.
\end{theorem}
\begin{theorem}
\label{th:uniqueness}
    Let $1\leq p <2$ and $p<\gamma$. If $X$ has finite $p-$variation and $f$ is $\text{Lip}(\gamma)$, then (\ref{eq:cde}) admits a unique solution for every $y_{0} \in W$ \cite{Lyons1994DIFFERENTIALED}.
\end{theorem}
These theorems extend the classic differential equation existence and uniqueness results to controls with unbounded variation but finite $p-$variation for $p<2$. A proof of these theorems is can be found in \cite{lyons2007differential}. These theorems are sufficient for the differential equations considered in this paper. However, there are many settings where the control has infinite $p-$variation for all $p<2$, such as Brownian motion. The theory of rough paths was developed in order to give meaning to (\ref{eq:cde}) when the control's $p-$variation is finite only for $p\geq2$ \cite{Lyons1998}. An introduction to rough path theory can be found in \cite{lyons2007differential}.

\subsection{The Tensor Algebra}
\label{sec:app_norm}

Let $V$ be a Banach space and $V^{\otimes n}$ denote the tensor powers of $V$, 
\begin{equation}
    V^{\otimes n} = V \underbrace{\otimes \cdots \otimes}_{\text{n-1\;times}} V.
\end{equation} 
There is choice in the norm of $V^{\otimes n}$. In this paper, we follow the setting of \citep{BOEDIHARDJO2016720} and \citep{lyons2002}. It is assumed that each $V^{\otimes n}$ is endowed with a norm such that the following conditions hold for all $v\in V^{\otimes n}$ and $w\in V^{\otimes m}$:
\begin{enumerate}
\item $||v|| = ||v_1\otimes \cdots \otimes v_n|| = ||v_{p(1)} \otimes \cdots \otimes v_{p(n)}||$ for all all bijective functions $p:\{1, \ldots, n\}\rightarrow \{1, \ldots, n\}$,
\item $|| v \otimes w || \leq ||v||\;||w||$,
\item for any bounded linear functional $f$ on $V^{\otimes n}$ and $g$ on $V^{\otimes m}$, there exists a unique bounded linear functional $f \otimes g$ on $V^{\otimes (m+n)}$ such that $(f \otimes g)(v \otimes w) = f(v)g(w). $
\end{enumerate}
\begin{definition} \emph{(The Tensor Algebra \citep{lyons2007differential})}
    For $n\geq1$, let $V^{\otimes n}$ be equipped with a norm satisfying the above conditions, and define $V^{\otimes 0}=\mathbb{R}$. The tensor algebra space is the set
    \begin{equation}
        T((V)) = \{\mathbf{x} = (x^0, x^1, \ldots ) | x^k \in V^{\otimes k}\}
    \end{equation}
    with product $\mathbf{z}=\mathbf{x}\otimes\mathbf{y}$ defined by
    \begin{equation}
        z^k = (\mathbf{x} \otimes \mathbf{y})^k = \sum_{j=0}^k x^j \otimes y^{k-j}.
    \end{equation}
\end{definition} 
The tensor algebra's product is associative and has unit $\mathbf{1} = (1,0,0,\ldots)$. As $T((V))$ is an associative algebra, it has a Lie algebra structure, with Lie bracket
\begin{equation}
[\mathbf{x}, \mathbf{y}] = \mathbf{x} \otimes \mathbf{y} - \mathbf{y} \otimes \mathbf{x}
\end{equation}
for $\mathbf{x},\mathbf{y} \in T((V))$ \citep{reutenauer1993free}.

\section{Proof of Theorem \ref{thm:lipnn}}
\label{app:lipnn_proof}

The proof of Theorem \ref{thm:lipnn} relies on two lemmas. The first is a bound on the $\text{Lip}(\gamma)-$norm of the composition of two $\mathrm{Lip}(\gamma)$ functions. The second is a bound on the $\mathrm{Lip}(2)-$norm of each layer of a fully connected neural network (FCNN).

\subsection{Composition of $\text{Lip}(\gamma)$ Functions}
\label{app:lipcomp}
\begin{lemma} \emph{(Composed $\text{Lip}(\gamma)-$norm \citep{composition})}
\label{lem:comp}
    Let $U$, $V$, and $W$ be Banach spaces and $\Sigma\subset U$ and $\Omega \subset V$ be closed. For $\gamma\geq 1$, let $f \in \text{Lip}(\gamma,\Sigma,\Omega)$ and $g \in \text{Lip}(\gamma,\Omega,W)$. Then the composition $g\circ f: \Omega \rightarrow W$ is $\text{Lip}(\gamma)$ with 
    \begin{equation}
    \label{eq:normcomp}
        ||g\circ f||_{\text{Lip}(\gamma)} \leq C_{\gamma}||g||_{\text{Lip}(\gamma)}\max\left\{||f||^{k+1}_{\text{Lip}(\gamma)},1\right\},
    \end{equation}
    where $k$ is the unique integer such that $\gamma\in(k, k+1]$ and $C_{\gamma}$ is a constant independent of $f$ and $g$.
\end{lemma}
The original statement of lemma \ref{lem:comp} in \citep{composition} gives \eqref{eq:normcomp} as
    \begin{equation}
    \label{eq:compboundincorrect}
        ||g\circ f||_{\text{Lip}(\gamma)} \leq C_{\gamma}||g||_{\text{Lip}(\gamma)}\max\left\{||f||^{\textcolor{red}{k}}_{\text{Lip}(\gamma)},1\right\}.
    \end{equation}
    We believe this is a small erratum, as for $g:[0,1]\rightarrow[0,1]$ defined as $g(x)=x$, \eqref{eq:compboundincorrect} implies there exists $C_1>0$ such that
    \begin{equation}
    ||g\circ f||_{\text{Lip}(1)} = ||f||_{\text{Lip}(1)} \leq C_1||g||_{\text{Lip}(1)}=C_1
    \end{equation}
    for all bounded and Lipschitz $f:[0,1]\rightarrow[0,1]$. As a counterexample, for any $C_1>0$, take $f(x)=x^{n}$ with $n>\max\{C_1, 1\}$. The following proof of lemma \ref{lem:comp} is given in \citep{composition}.
\begin{proof}
    Let $(g\circ f)^0, \ldots, (g\circ f)^k$ be defined by the generalisation of the chain rule to higher derivatives. Explicit calculation can be used to verify that if $f$ and $g$ are $\text{Lip}(\gamma)$, definition \ref{def:lipgamma} implies $g\circ f$ is $\text{Lip}(\gamma)$ with $||g\circ f||_{\text{Lip}(\gamma)}$ obeying \eqref{eq:normcomp}.
\end{proof}
Bounding the $\text{Lip}(\gamma)-$norm of a neural network (NN) requires an explicit form for $C_{\gamma}$ in \eqref{eq:normcomp}. This can be obtained via the explicit calculations mentioned in the proof of lemma \ref{lem:comp}. Here, we present the case $\gamma \in (1,2]$.
\begin{lemma} 
    \label{lem:normcomplip2}
    Let $U$, $V$, and $W$ be Banach spaces and $\Sigma\subset U$ and $\Omega \subset V$ be closed. For $\gamma \in (1,2]$, let $f = ( f^{(0)} , f^{(1)} ) \in \text{Lip}(\gamma,\Sigma,\Omega)$ and $g = (g^{(0)} , g^{(1)} ) \in \text{Lip}(\gamma,\Omega,W)$. Consider $h^{(0)} : \Sigma \to W$ and $h^{(1)} : \Sigma \to \mathbf{L}(V,W)$ defined for $p \in \Sigma$ and $v \in V$ by 
    \begin{equation}
    \label{eq:hdef}
    		h^{(0)}(p) := g^{(0)} \left( f^{(0)}(p) \right)
    		\qquad \text{and} \qquad 
    		h^{(1)}(p)[v] := g^{(1)}\left( f^{(0)}(p) \right) \left[ f^{(1)}(p)[v] \right].
    \end{equation}
    Then $h := \left( h^{(0)} , h^{(1)} \right) \in \text{Lip}(\gamma,\Sigma,W)$ and 
    \begin{equation}
    	\label{eq:normcomplip2}
    		|| h ||_{\text{Lip}(\gamma,\Sigma,W)} 
    		\leq 
    		\left( 1 + 2^{\gamma} \right) || g ||_{\text{Lip}(\gamma,\Omega,W)} \max \left\{ 1 , || f ||_{\text{Lip}(\gamma,\Sigma,\Omega)}^{\gamma} \right\}. 
    \end{equation}
\end{lemma}
\begin{proof}
From definition \ref{def:lipgamma}, $f^{(0)} : \Sigma \to \Omega$, $f^{(1)} : \Sigma \to \mathbf{L}(U,V)$, $g^{(0)} : \Omega \to W$ and $g^{(1)} : \Omega \to \mathbf{L}(V,W)$. Furthermore, for all $p\in\Sigma$
\begin{equation}
	\label{f_pointwise_bounds}
		(\bI) \quad \left|\left| f^{(0)}(p) \right|\right|_V \leq ||f||_{\text{Lip}(\gamma,\Sigma,\Omega)}
		\qquad \text{and} \quad 
		(\bII) \quad \left|\left| f^{(1)}(p) \right|\right|_{\mathbf{L}(U,V)} \leq ||f||_{\text{Lip}(\gamma,\Sigma,\Omega)}.
\end{equation}
Similarly, for all $x \in \Omega$ we have that
\begin{equation}
	\label{g_pointwise_bounds}
		(\bI) \quad \left|\left| g^{(0)}(x) \right|\right|_W \leq ||g||_{\text{Lip}(\gamma,\Omega,W)}
		\qquad \text{and} \quad 
		(\bII) \quad \left|\left| g^{(1)}(x) \right|\right|_{\mathbf{L}(V,W)} \leq ||g||_{\text{Lip}(\gamma,\Omega,W)}.
\end{equation}
Define 
$R^f_0 : \Sigma \times \Sigma \to V$ and $R^f_1 : \Sigma \times \Sigma \to \mathbf{L}(U,V)$ by
\begin{equation}
	\label{f_remainder_term_defs}
    \begin{aligned}
		R^f_0 (p,q) &:= f^{(0)}(q) - f^{(0)}(p) - f^{(1)}(p)[q-p], \\
		R^f_1(p,q)[u] &:= f^{(1)}(q)[u] - f^{(1)}(p)[u],
  \end{aligned}
\end{equation}
for any $p,q \in \Sigma$ and $u \in U$. Then
\begin{equation}
	\label{f_remainder_term_bounds}
		\begin{aligned}
			&	(\bI) \quad \left|\left| R^f_0(p,q) \right|\right|_V \leq ||f||_{\text{Lip}(\gamma,\Sigma,\Omega)} ||q-p||_U^{\gamma}, \\
			%\qquad \text{and} \qquad
			&	(\bII) \quad \left|\left| R^f_1(p,q) \right|\right|_{\mathbf{L}(U,V)} \leq ||f||_{\text{Lip}(\gamma,\Sigma,\Omega)} ||q-p||_U^{\gamma - 1}.
		\end{aligned}
\end{equation}
Similarly, define
$R^g_0 : \Omega \times \Omega \to W$ and $R^g_1 : \Omega \times \Omega \to \mathbf{L}(V,W)$ by
\begin{equation}
	\label{g_remainder_term_defs}
    \begin{aligned}
		R^g_0 (x,y) &:= g^{(0)}(y) - g^{(0)}(x) - g^{(1)}(x)[y-x], \\
		R^g_1(x,y)[v] &:= g^{(1)}(y)[v] - g^{(1)}(x)[v],
  \end{aligned}
\end{equation}
for $x,y \in \Omega$ and $v \in V$. Then,
\begin{equation}
	\label{g_remainder_term_bounds}
		\begin{aligned}
			&	(\bI) \quad \left|\left| R^g_0(x,y) \right|\right|_W \leq || g ||_{\text{Lip}(\gamma,\Omega,W)} ||y-x||_V^{\gamma} \\
			%\qquad \text{and} \qquad
			&	(\bII) \quad \left|\left| R^g_1(x,y) \right|\right|_{\mathbf{L}(V,W)} \leq || g ||_{\text{Lip}(\gamma,\Omega,W)} ||y-x||_V^{\gamma - 1}.
		\end{aligned}
\end{equation}
Define $h^{(0)} : \Sigma \to W$ and $h^{(1)} : \Sigma \to \mathbf{L}(V,W)$ as in \eqref{eq:hdef},
\begin{equation}
	\label{lip_gamma_chain_rule_h_def_proof}
		h^{(0)}(p) := g^{(0)} \left( f^{(0)}(p) \right)
		\qquad \text{and} \qquad 
		h^{(1)}(p)[u] := g^{(1)}\left( f^{(0)}(p) \right) \left[ f^{(1)}(p)[u] \right],
\end{equation}
for $p \in \Sigma$ and $u \in U$. Finally, define remainder terms $R^h_0 : \Sigma \times \Sigma \to W$ and $R^h_1 : \Sigma \times \Sigma \to \mathbf{L}(U,W)$ by 
\begin{equation}
	\label{h_remain_terms_def}
 \begin{aligned}
		R^h_0(p,q) &:= h^{(0)}(q) - h^{(0)}(p) - h^{(1)}(p)[q-p], \\
		R^h_1(p,q)[u] &:= h^{(1)}(q)[u] - h^{(1)}(p)[u],
  \end{aligned}
\end{equation}
for $p,q \in \Sigma$ and $u \in U$. We now establish that $h = ( h^{(0)} , h^{(1)} ) \in \text{Lip}(\gamma,\Sigma,W)$ and that the norm estimate claimed in 
\eqref{eq:normcomplip2} is satisfied.

First we consider the bounds on $h^{(0)}$ and $h^{(1)}$. For any $p \in \Sigma$,   
(\bI) in \eqref{g_pointwise_bounds} implies that
\begin{equation}
	\label{h0_bd}
		\left|\left| h^{(0)}(p) \right|\right|_W 
		= 
		\left|\left| g^{(0)} \left( f^{(0)}(p) \right)  \right|\right|_W
		\leq %\stackrel{(\bI) in \eqref{g_pointwise_bounds}}{\leq}
		|| g ||_{\text{Lip}(\gamma,\Omega,W)}
\end{equation}
since $f^{(0)}(p) \in \Omega$. Further, for any $p \in \Sigma$ and any $u \in U$, \eqref{g_pointwise_bounds} and (\bII) in \eqref{f_pointwise_bounds} imply that
\begin{align*}
	\left|\left| h^{(1)}(p)[u] \right|\right|_W &= 
		\left|\left| g^{(1)} \left( f^{(0)}(p) \right) \left[ f^{(1)}(p)[u] \right] \right|\right|_W \\
		&\leq \left|\left| g^{(1)} \left( f^{(0)}(p) \right) \right|\right|_{\mathbf{L}(V,W)} \left|\left| f^{(1)}(p) \right|\right|_{\mathbf{L}(U,V)} ||u||_U \\
		&\leq
		%\stackrel{(\bII) in \eqref{f_pointwise_bounds} and \eqref{g_pointwise_bounds} }{\leq}
		|| g ||_{\text{Lip}(\gamma,\Omega,W)} || f ||_{\text{Lip}(\gamma,\Sigma,\Omega)} ||u||_U
\end{align*}
since $f^{(0)}(p) \in \Omega$. Taking the supremum over $u \in U$ with unit $U$-norm, it follows that
\begin{equation}
	\label{h1_bd}
		\left| \left| h^{(1)}(p) \right|\right|_{\mathbf{L}(U,W)} \leq || g ||_{\text{Lip}(\gamma,\Omega,W)} || f ||_{\text{Lip}(\gamma,\Sigma,\Omega)}.
\end{equation}
Now we consider the bounds on $R^h_0$ and $R^h_1$. For this purpose we fix $p,q \in \Sigma$ and $u \in U$. We first assume that $||q - p||_U > 1$. 
In this case we may use \eqref{h0_bd} and \eqref{h1_bd} to compute that
\begin{align*}
	\left|\left| R^h_0(p,q) \right|\right|_W &= \left| \left| h^{(0)}(q) - h^{(0)}(p) - h^{(1)}(p)[q-p] \right|\right|_W \\
		%&\leq \left|\left| h^{(0)}(q) \right|\right|_W + \left|\left| h^{(0)}(p) \right|\right|_W +
		%\left|\left| h^{(1)}(p) \right|\right|_{\mathbf{L}(V,W)} || q - p ||_U \\
		&\leq%\stackrel{\eqref{h0_bd} and \eqref{h1_bd}}{\leq} 
		2 || g ||_{\text{Lip}(\gamma,\Omega,W)} + || g ||_{\text{Lip}(\gamma,\Omega,W)} || f ||_{\text{Lip}(\gamma,\Sigma,\Omega)} || q - p ||_U.
\end{align*}
Since $\gamma > 1$ means that $1 < || q - p ||_U < || q  - p ||_U^{\gamma}$, we deduce that
\begin{equation}
	\label{Rh0_bd_a}
		\left| \left| R^h_0(p,q) \right|\right|_W \leq 
		|| g ||_{\text{Lip}(\gamma,\Omega,W)} \left( 2 + || f ||_{\text{Lip}(\gamma,\Sigma,\Omega)} \right) || q - p ||_U^{\gamma}.
\end{equation}
Similarly, we may use \eqref{h1_bd} and that $ ||q-p||_U^{\gamma - 1} > 1$ to compute that
\begin{equation}
	\label{Rh1_bd_a_with_v}
	\left|\left| R^h_1(p,q)[u] \right|\right|_W = \left|\left| h^{(1)}(q)[u] - h^{(1)}(p)[u] \right|\right|_W 
		%&\leq \left( \left|\left| h^{(1)}(q) \right|\right|_{\mathbf{L}(V,W)} + \left|\left| h^{(1)}(q) \right|\right|_{\mathbf{L}(V,W)} \right) ||u||_U \\
		\leq%\stackrel{\eqref{h1_bd}}{\leq}
		%2 || g ||_{\text{Lip}(\gamma,\Omega,W)} || f ||_{\text{Lip}(\gamma,\Sigma,\Omega)} ||u||_U \\
		2 || g ||_{\text{Lip}(\gamma,\Omega,W)} || f ||_{\text{Lip}(\gamma,\Sigma,\Omega)} || q - p ||_U^{\gamma - 1} ||u||_U.
\end{equation}
%where the last line has used that $1 < || q - p ||_U^{\gamma - 1}$ since $\gamma > 1$ and $|| q - p ||_U > 1$.
Taking the supremum over $u \in U$ with unit $U$-norm in \eqref{Rh1_bd_a_with_v} yields the estimate that
\begin{equation}
	\label{Rh1_bd_a}
		\left|\left| R^h_1(p,q) \right|\right|_{\mathbf{L}(V,W)} 
		\leq
		2 || g ||_{\text{Lip}(\gamma,\Omega,W)} || f ||_{\text{Lip}(\gamma,\Sigma,\Omega)} || q - p ||_U^{\gamma - 1}.
\end{equation}
Together, \eqref{Rh0_bd_a} and \eqref{Rh1_bd_a} establish the remainder term estimates required to conclude that $h = (h^{(0)} , h^{(1)} ) \in \text{Lip}(\gamma,\Sigma,W)$
in the case that $|| q - p ||_U > 1$. 
\\ \\
We next establish similar remainder term estimates when $|| q - p ||_U < 1$. Thus we fix $p,q \in \Sigma$ and assume that $||q-p||_U < 1$.
Note that $\gamma > 1$ means that $||q - p||_U^{\gamma} < ||q-p||_U < 1$. Additionally,
\begin{equation}
\label{f0_diff_bd}
    \begin{aligned}
        \left|\left| f^{(0)}(q) - f^{(0)}(p) \right|\right|_V &\stackrel{(\ref{f_remainder_term_defs})}{=} \left|\left|f^{(1)}(p)[q-p] + R^f_0(p,q) \right|\right|_V, \\
        &\;\;\leq || f ||_{\text{Lip}(\gamma,\Sigma,\Omega)} \left( ||q-p||_U + ||q-p||_U^{\gamma} \right), \\
		&\;\;\leq 2 || f ||_{\text{Lip}(\gamma,\Sigma,\Omega)} ||q-p||_U,
    \end{aligned}
\end{equation}
where (\bII) in \eqref{f_pointwise_bounds} and (\bI) in \eqref{f_remainder_term_bounds} have been used. We now consider the term $R^h_0(p,q)$. We start by observing that 
\begin{align*}
	R^h_0(p,q) &\stackrel{(\ref{h_remain_terms_def})}{=} h^{(0)}(q) - h^{(0)}(p) - h^{(1)}(p)[q-p] \\
		&\stackrel{(\ref{lip_gamma_chain_rule_h_def_proof})}{=} 
		g^{(0)}\left( f^{(0)}(q) \right) - g^{(0)} \left( f^{(0)}(p) \right) - g^{(1)}\left(f^{(0)}(p) \right) \left[ f^{(1)}(p)[q-p] \right]  \\
		&\stackrel{(\ref{g_remainder_term_defs})}{=}
		g^{(1)}\left( f^{(0)}(p) \right) \left[ f^{(0)}(q) - f^{(0)}(p) -  f^{(1)}(p)[q-p] \right] + R^g_0\left( f^{(0)}(p) , f^{(0)}(q) \right) \\
		&\stackrel{(\ref{f_remainder_term_defs})}{=}
		g^{(1)}\left( f^{(0)}(p) \right) \left[ R^f_0(p,q) \right] + R^g_0\left( f^{(0)}(p) , f^{(0)}(q) \right).
\end{align*}
Consequently, by using (\bII) in \eqref{g_pointwise_bounds} to estimate the term $g^{(1)}\left( f^{(0)}(p) \right)$, 
(\bI) in \eqref{f_remainder_term_bounds} to estimate the term $R^f_0(p,q)$, and 
(\bI) in \eqref{g_remainder_term_bounds} to estimate the term $R^g_0\left( f^{(0)}(p) , f^{(0)}(q) \right)$, 
we may deduce that
\begin{equation}
	\label{Rh0_bd_b_almost}
		\left|\left| R^h_0(p,q) \right|\right|_W \leq 
		|| g ||_{\text{Lip}(\gamma,\Omega,W)} \left( || f ||_{\text{Lip}(\gamma,\Sigma,\Omega)} || q - p ||_U^{\gamma} + 
		\left|\left| f^{(0)}(q) - f^{(0)}(p) \right|\right|_V^{\gamma} \right).
\end{equation}
The combination of \eqref{f0_diff_bd} and \eqref{Rh0_bd_b_almost} yields the estimate 
\begin{equation}
	\label{Rh0_bd_b}
		\left|\left| R^h_0(p,q) \right|\right|_W \leq 
		|| g ||_{\text{Lip}(\gamma,\Omega,W)} \left( || f ||_{\text{Lip}(\gamma,\Sigma,\Omega)}  + 
		2^{\gamma} ||f||_{\text{Lip}(\gamma,\Sigma,\Omega)}^{\gamma} \right) || q - p ||_U^{\gamma}.
\end{equation}
Turning our attention to $R^h_1$, we fix $u \in U$ and compute that
\begin{align*}
	R^h_1(p,q)[u] &\stackrel{(\ref{h_remain_terms_def})}{=} h^{(1)}(q)[u] - h^{(1)}(p)[u]  \\
		&\stackrel{(\ref{lip_gamma_chain_rule_h_def_proof})}{=}  
		g^{(1)}\left( f^{(0)}(q) \right) \left[ f^{(1)}(q)[u] \right] - g^{(1)}\left( f^{(0)}(p) \right) \left[ f^{(1)}(p)[u] \right] \\
		&\stackrel{(\ref{g_remainder_term_defs})}{=} 
		g^{(1)} \left( f^{(0)}(p) \right) \left[ f^{(1)}(q)[u] - f^{(1)}(p)[u] \right] + R^g_1 \left( f^{(0)}(p) , f^{(0)}(q) \right) \left[ f^{(1)}(q)[u] \right]  \\
		&\stackrel{(\ref{f_remainder_term_defs})}{=} 
		g^{(1)} \left( f^{(0)}(p) \right) \left[ R^f_1(p,q)[u] \right] + R^g_1 \left( f^{(0)}(p) , f^{(0)}(q) \right) \left[ f^{(1)}(q)[u] \right].
\end{align*}
Consequently, by using (\bII) in \eqref{g_pointwise_bounds} to estimate the term $g^{(1)} \left( f^{(0)}(p) \right)$,
(\bII) in \eqref{f_pointwise_bounds} to estimate the term $f^{(1)}(q)$,
(\bII) in \eqref{f_remainder_term_bounds} to estimate the term $R^f_1(p,q)$,  and 
(\bII) in \eqref{g_remainder_term_bounds} to estimate the term $R^g_1 \left( f^{(0)}(p) , f^{(0)}(q) \right)$, 
we may deduce that
\begin{equation}
	\label{Rh1_bd_b_almost}
		\left|\left| R^h_1(p,q)[u] \right|\right|_W \leq 
		|| g ||_{\text{Lip}(\gamma,\Omega,W)} || f ||_{\text{Lip}(\gamma,\Sigma,\Omega)} \left( || q - p ||_U^{\gamma - 1}
		+ 
		\left|\left| f^{(0)}(q)  - f^{(0)}(p) \right|\right|_{V}^{\gamma - 1}  \right) ||u||_U.
\end{equation}
The combination of \eqref{f0_diff_bd} and \eqref{Rh1_bd_b_almost} yields the estimate that
\begin{equation}
	\label{Rh1_bd_b_almost2}
		\left|\left| R^h_1(p,q)[u] \right|\right|_W \leq 
		|| g ||_{\text{Lip}(\gamma,\Omega,W)} \left( || f ||_{\text{Lip}(\gamma,\Sigma,\Omega)} 
		+ 2^{\gamma - 1}||f||_{\text{Lip}(\gamma,\Sigma,\Omega)}^{\gamma}   \right) ||q - p||_U^{\gamma - 1} ||u||_U.
\end{equation}
Taking the supremum over $u \in U$ with unit $U$-norm in \eqref{Rh1_bd_b_almost2} yields the estimate that
\begin{equation}
	\label{Rh1_bd_b}
		\left|\left| R^h_1(p,q) \right|\right|_{\mathbf{L}(V,W)} \leq 
		|| g ||_{\text{Lip}(\gamma,\Omega,W)} \left( || f ||_{\text{Lip}(\gamma,\Sigma,\Omega)}  + 
		2^{\gamma - 1} ||f||_{\text{Lip}(\gamma,\Sigma,\Omega)}^{\gamma} \right) || q - p ||_U^{\gamma - 1}.
\end{equation}
Finally, we complete the proof by combining the various estimates we have established for $h$ to obtain the $\text{Lip}(\gamma,\Sigma,W)$-norm bound 
claimed in \eqref{eq:normcomplip2}.

We start this task by combining \eqref{Rh0_bd_a} and \eqref{Rh0_bd_b} to deduce that for every $p,q \in \Sigma$ we have 
\begin{equation}
	\label{Rh0_bd_c}
    \begin{aligned}
		\left|\left| R^h_0(p,q) \right|\right|_W \leq 
		\twopartdef{|| g ||_{\text{Lip}(\gamma,\Omega,W)} \left( 2 + || f ||_{\text{Lip}(\gamma,\Sigma,\Omega)} \right) || q - p ||_U^{\gamma}}{||q-p||_U > 1}
		{|| g ||_{\text{Lip}(\gamma,\Omega,W)} \left( || f ||_{\text{Lip}(\gamma,\Sigma,\Omega)}  + 
		2^{\gamma} ||f||_{\text{Lip}(\gamma,\Sigma,\Omega)}^{\gamma} \right) || q - p ||_U^{\gamma}}{||q - p ||_U \leq 1.}
  \end{aligned}
\end{equation}
Moreover, the combination of \eqref{Rh1_bd_a} and \eqref{Rh1_bd_b} yields the estimate that
\begin{equation}
\small
	\label{Rh1_bd_c}
		\left|\left| R^h_1(p,q) \right|\right|_{\mathbf{L}(V,W)} \leq 
		\twopartdef{2 || g ||_{\text{Lip}(\gamma,\Omega,W)} || f ||_{\text{Lip}(\gamma,\Sigma,\Omega)} || q - p ||_U^{\gamma - 1}}{||q-p||_U > 1}
		{|| g ||_{\text{Lip}(\gamma,\Omega,W)} \left( 
        || f ||_{\text{Lip}(\gamma,\Sigma,\Omega)}  + 
		2^{\gamma - 1} ||f||_{\text{Lip}(\gamma,\Sigma,\Omega)}^{\gamma} \right) || q - p ||_U^{\gamma - 1}}{||q - p ||_U \leq 1.}
\end{equation}
A consequence of \eqref{Rh0_bd_c} is that 
\begin{equation}
	\label{Rh0_bd_d}
		\left|\left| R^h_0(p,q) \right|\right|_W \leq \left( 1 + 2^{\gamma} \right)|| g ||_{\text{Lip}(\gamma,\Omega,W)} 
		\max \left\{||f||_{\text{Lip}(\gamma,\Sigma,\Omega)}^{\gamma} , 1 \right\} ||q - p||_U^{\gamma},
\end{equation}
whilst a consequence of \eqref{Rh1_bd_c} is that 
\begin{equation}
	\label{Rh1_bd_d}
		\left|\left| R^h_1(p,q) \right|\right|_{\mathbf{L}(V,W)}\leq \left( 1 + 2^{\gamma - 1} \right) || g ||_{\text{Lip}(\gamma,\Omega,W)} 
		\max \left\{||f||_{\text{Lip}(\gamma,\Sigma,\Omega)}^{\gamma}, 1 \right\} ||q - p||_U^{\gamma - 1}.
\end{equation}
Therefore, by combining \eqref{h0_bd}, \eqref{h1_bd}, \eqref{Rh0_bd_d}, and \eqref{Rh1_bd_d}, we conclude both that $h = (h^{(0)} , h^{(1)} ) \in \text{Lip}(\gamma,\Sigma,W)$ 
and that
\begin{equation}
	\label{h_lip_gamma_norm}
		|| h ||_{\text{Lip}(\gamma,\Sigma,W)} \leq 
        \left( 1 + 2^{\gamma} \right) 
        || g ||_{\text{Lip}(\gamma,\Omega,W)} 
		\max \left\{ 
        ||f||^{\gamma}_{\text{Lip}(\gamma,\Sigma,\Omega)}, 1 \right\}.
\end{equation}
\end{proof}
Note that (\ref{eq:normcomplip2}) is a stricter bound than (\ref{eq:normcomp}), as for $\gamma\in(k, k+1]$,
\begin{equation}
    \max \left\{|| f ||_{\text{Lip}(\gamma,\Sigma,\Omega)}^{\gamma}, 1 \right\} \leq \max\left\{||f||^{k+1}_{\text{Lip}(\gamma)},1\right\}.
\end{equation}
There is equality when $||f||_{\text{Lip}(\gamma)}\leq 1$ or $\gamma=2$. 

\subsection{$\text{Lip}(2)-$norm of a Neural Network Layer}

Lemma \ref{lem:normcomplip2} allows us to bound the $\text{Lip}(2)-$norm of a neural network (NN) given a bound on the $\text{Lip}(2)-$norm of each layer of a NN. We demonstrate this here for a simple NN.
\begin{definition} \emph{(Fully Connected NN)}
\label{def:FCNN}
    Let $m,n_{in},n_{out},n_h\in\mathbb{N}$ and $f_{\theta}$ be a fully connected NN with $m$ layers, input dimension $n_{in}$, output dimension $n_{out}$, hidden dimension $n_h$, and activation function $\sigma$. Given an input $\mathbf{x}\in\mathbb{R}^{n_{in}}$, the output of the NN is defined as
    \begin{equation}
        \mathbf{y} = L^m(\cdots (L^1(\mathbf{x}))\cdots),
    \end{equation}
    where $L^1:\mathbb{R}^{n_{in}}\rightarrow\mathbb{R}^{n_h}$, $L^i:\mathbb{R}^{n_h}\rightarrow\mathbb{R}^{n_h}$ for $i=2,\ldots,m-1$, and $L^m:\mathbb{R}^{n_h}\rightarrow\mathbb{R}^{n_{out}}$. Each layer is defined by
    \begin{equation}
        L^i(\mathbf{y}) = \begin{bmatrix} L^i_1(\mathbf{y}) \\ \vdots \\ L^i_\alpha(\mathbf{y}) \end{bmatrix} = \begin{bmatrix} \sigma(l^i_1(\mathbf{y})) \\ \vdots \\ \sigma(l^i_\alpha(\mathbf{y}) )\end{bmatrix} = \begin{bmatrix} \sigma(W^i_1 \cdot \mathbf{y} + b^i_1) \\ \vdots \\ \sigma(W^i_\alpha \cdot \mathbf{y} + b^i_\alpha) \end{bmatrix},
    \end{equation}
    where $\mathbf{y}\in\mathbb{R}^\beta$, $W^i=[W^i_1, \ldots, W^i_\alpha]^T\in\mathbb{R}^{\alpha\times \beta}$ and $\mathbf{b}^i =[b^i_1, \ldots, b^i_\alpha]^T\in\mathbb{R}^\alpha$ are the learnable parameters and
    \begin{equation}
        (\alpha,\beta) = \begin{cases} (n_h,n_{in}), \qquad i=1, \\
        (n_h, n_h), \qquad \; i=2,\ldots,m-1, \\ (n_{out}, n_h), \qquad i=m.
        \end{cases}
    \end{equation}
\end{definition}
\begin{definition} \emph{(SiLU \citep{ELFWING20183})}
    The activation function $\text{\emph{SiLU}}:\mathbb{R}\rightarrow\mathbb{R}$ is defined as
\begin{equation}
    \text{\emph{SiLU}}(y) = \frac{y}{1+e^{-y}}.
\end{equation}
\end{definition}
\begin{lemma}
\label{lem:lip2layer}
    Let $f_{\theta}$ be a fully connected NN with activation function $\text{SiLU}$. Assume the input is normalised such that $\mathbf{x}=[x_1,\ldots,x_{n_{in}}]^T$ satisfies $|x_j|\leq1$ for $j=1,\ldots,n_{in}$. Then 
    \begin{equation}
        ||L^i||_{\text{Lip}(2)} \leq \left(\sum_{j=1}^n\max\left\{0.5||W^i_j||_2^2, 1.1||W^i_j||_2, \Gamma^i ||W^i_j||_2 + |b^i_j|\right\}^2\right)^{\frac{1}{2}},
    \end{equation}
    where
    \begin{equation}
        \Gamma^i = \sqrt{n_h}||W^{i-1}||_2\bigg(\cdots\Big(\sqrt{n_h}||W^{2}||_2\big(\sqrt{n_{in}}||W^1||_2+||\mathbf{b}^1||_2\big)+||\mathbf{b}^{2}||_2\Big)+\cdots\bigg)+||\mathbf{b}^{i-1}||_2.
    \end{equation}
\end{lemma}
\begin{proof} 
    Consider $L^1:A\rightarrow\mathbb{R}^n$, where $A\subset\mathbb{R}^{n_{in}}$ is the set of $\mathbf{x}=[x_1,\ldots,x_{n_{in}}]^T$ satisfying $|x_j|\leq1$ for $j=1,\ldots,n_{in}$. Each $L^1_j:\mathbb{R}^{n_{in}}\rightarrow\mathbb{R}$ is composed of a linear layer $l^1_j$ and a $\text{SiLU}$. 
    First, we show that
    \begin{equation}
        (L^1_j, \nabla L^1_j) = \left(\frac{l^1_j}{1+e^{-l^1_j}}, \frac{e^{l^1_j}(e^{l^1_j}+l^1_j+1)}{(e^{l^1_j}+1)^2}W^1_j\right) \in \text{Lip}(2),
    \end{equation}
    where $l^1_j(\mathbf{x})=W^1_j\cdot\mathbf{x}+b^1_j.$ So,
    \begin{equation}
        \max_{\mathbf{x}\in A}|L_j^1(\mathbf{x}) |\leq \max_{\mathbf{x}\in A}|l_j^1(\mathbf{x})| \leq \sqrt{n_{in}}||W^1_j||_2+|b^1_j|
    \end{equation}
    and 
    \begin{equation}
        ||\nabla L^1_j||_2 \leq 1.1||W^1_j||_2.
    \end{equation}
    Since $L^1_j$ is at least twice differentiable, 
    \begin{equation}
        \frac{|L^1_j(\mathbf{y})-L^1_j(\mathbf{x}) - \nabla L^1_j(\mathbf{x})(\mathbf{y}-\mathbf{x})|}{||\mathbf{y}-\mathbf{x}||_2^2} = \frac{|(\mathbf{y}-\mathbf{x})^T\nabla^2 L^1_j(\mathbf{b})(\mathbf{y}-\mathbf{x}))|}{2|\mathbf{x}-\mathbf{y}|_2^2}\leq \frac{1}{2}\beta(\nabla^2 L^1_j(\mathbf{b})),
    \end{equation}
    where $\mathbf{b} = \mathbf{x} + t(\mathbf{y}-\mathbf{x})$ for some $t\in(0, 1)$ and $\beta(A) =  \max\{|\lambda_{\max}(A)|, |\lambda_{\min}(A)|\}$. Similarly,
    \begin{equation}
        \frac{||\nabla L^1_j(\mathbf{y})-\nabla L^1_j(\mathbf{x})||_2}{||\mathbf{y}-\mathbf{x}||_2} = \frac{|\nabla^2 L^1_j(\mathbf{z})(\mathbf{y}-\mathbf{x})|_2}{|\mathbf{x}-\mathbf{y}|_2} \leq ||\nabla^2 L^1_j(\mathbf{z})||_2 = \beta(\nabla^2 L^1_j(\mathbf{z})), 
    \end{equation}
    where $\mathbf{z} = \mathbf{x} + \alpha(\mathbf{y}-\mathbf{x})$ for some $\alpha\in(0, 1)$. Now, \begin{equation}
        \nabla^2 L^1_j = \frac{e^{l^1_j}(2 + 2e^{l^1_j}+l^1_j(1-e^{l^1_j})}{(e^{l^1_j}+1)^3}W^1_j \otimes W^1_j,
    \end{equation}
    which has one non-zero eigenvalue 
    \begin{equation}
        \beta(\nabla^2 L^1_j) = \frac{e^{l^1_j}(2 + 2e^{l^1_j}+l^1_j(1-e^{l^1_j})}{(e^{l^1_j}+1)^3}||W^1_j||_2^2
    \end{equation}
    satisfying
    \begin{equation}
        \max_{\mathbf{x}\in\mathbb{R}^{n_{in}}}\beta(\nabla^2 L^1_j(\mathbf{x})) = 0.5||W^1_j||_2^2.
    \end{equation}
    Therefore, 
    \begin{equation}
        ||L^1_j||_{\text{Lip}(2)} = \max\left\{0.5||W^1_j||_2^2, 1.1||W^1_j||_2, \sqrt{n_{in}}||W^1_j||_2+|b^1_j|\right\},
    \end{equation}
    and
    \begin{equation}
        ||L^1||_{\text{Lip}(2)} = \left(\sum_{j=1}^n\max\left\{0.5||W^1_j||_2^2, 1.1||W^1_j||_2, \sqrt{n_{in}}||W^1_j||_2+|b^1_j|\right\}^2\right)^{\frac{1}{2}}.
    \end{equation}
    The calculations for subsequent layers are very similar, except that the input to each layer is no longer restricted to $A$. For example,
    \begin{equation}
        \begin{aligned}
        \max_{\mathbf{x}\in A}|L^2_j(L^1(\mathbf{x}))| &= \max_{\mathbf{x}\in A}|W^2_j\cdot\text{SiLU}(W^1\mathbf{x}+\mathbf{b}^1)+b^2_j|, \\
        &\leq \left( \sqrt{n_{in}}||W^1||_2+||\mathbf{b}^1||_2\right)||W^2_j||_2+|b^2_j|.
        \end{aligned}
    \end{equation}
    In general
    \begin{equation}
        ||L^i||_{\text{Lip}(2)} \leq \left(\sum_{j=1}^n\max\left\{0.5||W^i_j||_2^2, 1.1||W^i_j||_2, \Gamma^i ||W^i_j||_2 + |b^i_j|\right\}^2\right)^{\frac{1}{2}},
    \end{equation}
    where
    \begin{equation}
        \Gamma^i = \sqrt{n_h}||W^{i-1}||_2\bigg(\cdots\Big(\sqrt{n_h}||W^{2}||_2\big(\sqrt{n_{in}}||W^1||_2+||\mathbf{b}^1||_2\big)+||\mathbf{b}^{2}||_2\Big)+\cdots\bigg)+||\mathbf{b}^{i-1}||_2.
    \end{equation}
\end{proof}

\subsection{Proof of Theorem \ref{thm:lipnn}}

\begin{theorem}
\label{thm:app_lipnn}
    Let $f_{\theta}$ be a FCNN with input dimension $n_{in}$, hidden dimension $n_h$, depth $m$, and activation function $\text{SiLU}$ \citep{ELFWING20183}. Assuming the input $\mathbf{x}=[x_1,\ldots,x_{n_{in}}]^T$ satisfies $|x_j|\leq1$ for $j=1,\ldots,n_{in}$, then $f_{\theta}\in\mathrm{Lip}(2)$ and 
    \begin{equation}
    \label{eq:lip2nnbound_app}
        ||f_{\theta}||_{\text{Lip}(2)} \leq CP_{m!}(||W^1||_2,\ldots,||W^m||_2,||\mathbf{b}^1||_2,\ldots,||\mathbf{b}^m||_2)
    \end{equation}
    where $C$ is a constant depending on $n_{in}, n_h,$ and $m$, $\{W^i\}_{i=1}^m$ and $\{\mathbf{b}^i\}_{i=1}^m$ are the weights and biases of $i^{\text{th}}$ layer of $f_{\theta}$, and $P_{m!}$ is a polynomial of order $m!$.
\end{theorem}

\begin{proof}
    Lemma \ref{lem:lip2layer} means that each layer $L^j$ of $f_{\theta}$ is $\mathrm{Lip}(2)$ with norm satisfying
    \begin{equation}
        ||L^j||_{\text{Lip}(2)} \leq CP_j(||W_1||_2,\ldots,||W_j||_2,||\mathbf{b}_1||_2,\ldots,||\mathbf{b}_j||_2),
    \end{equation}
    where $C$ is a constant depending on $n$ and $n_{in}$ and $P_j$ is a $j^{\text{th}}$ order polynomial. Applying lemma \ref{lem:normcomplip2} gives the bound in \eqref{eq:lip2nnbound_app}.
\end{proof}

\section{Experimental Details}

\subsection{Stacked Recurrent Models}
\label{app:baseline_models}

The stacked recurrent models considered in this paper are based on the official implementation of S5 located at \url{https://github.com/lindermanlab/S5} \citep{S5}. A recurrent block consists of a batch or layer normalisation, a recurrent layer, a GLU layer \citep{dauphin2017language}, dropout with rate $0.1$, and a skip connection. A full model consists of a linear encoder, a number of stacked recurrent blocks, and a final linear layer. The four different recurrent layers considered are the linear recurrent unit, S5, S6, and MAMBA, where S6 refers to the selective state-space recurrence introduced by \citet{gu2023mamba} and MAMBA refers to the combination of a gated MLP, convolution, and S6 recurrence that was also introduced by \citet{gu2023mamba}. S5 and LRU use batch normalisation, whereas S6 and MAMBA use layer normalisation.

\subsection{CDE Models}
\label{app:cde}

NCDEs, NRDEs, and Log-NCDEs use a single linear layer as $\xi_{\phi}$. NCDEs and NRDEs use FCNNs as their vector fields configured in the same way as their original papers \citep{kidger2020neuralcde, morrill2021neuralrough}. NCDEs use $\text{ReLU}$ activation functions for the hidden layers and a final activation function of $\tanh$. NRDEs use the same, but move the $\tanh$ activation function to be before the final linear layer in the FCNN. Log-NCDEs use a FCNN with SiLU activation functions for the hidden layers and a final activation function of $\tanh$. NRDEs and Log-NCDEs take their intervals $r_{i+1} - r_i$ to be a fixed number of observations, referred to as the Log-ODE step. 

\subsection{Toy Dataset Details}
\label{app:toy_details}

On the toy dataset, all models use a hidden state of dimension $64$ and Adam with a learning rate of $0.0003$ \citep{kingma2017adam}. LRU, S5, S6, and MAMBA use $6$ blocks and S5, S6, and MAMBA use a state dimension of $64$. S5 uses $2$ initialisation blocks and MAMBA uses a convolution dimension of $4$ and an expansion factor of $2$. NCDEs, NRDEs, and Log-NCDEs use a FCNN with width $128$ and depth $3$ as their vector field. Furthermore, they all use Heun as their differential equation solver with a fixed stepsize of $0.01$. NRDEs and Log-NCDEs use a Log-ODE step of $4$ and a signature truncation depth of $2$. Log-NCDEs do not use any $\mathrm{Lip}(\gamma)$ regularisation, i.e. $\lambda=0$.

\subsection{UEA-MTSCA and PPG-DaLiA Details}
\label{app:uea_ppg_details}

Table \ref{tab:UEA_summary} provides details on the dimension, number of observations, and number of classes for the datasets chosen from the UEA-MTSCA for the experiments conducted in this paper. Care is necessary when using the EigenWorms dataset. As of June $1^{\text{st}}$ 2024, the train and test splits obtained from \url{https://timeseriesclassification.com/description.php?Dataset=EigenWorms} contain repeated time series. The repeated data was removed for the experiments in this paper.
\begin{table}
\centering
\caption{A summary of the subset of the UEA-MTSCA datasets used in this paper.}
\vspace{0.2cm}
\begin{tabular}{l|c c c c}
Dataset & Dimension & Number of Observations & Classes \\ \hline
EigenWorms & $6$ & $17984$ & $5$ \\
EthanolConcentration & $3$ & $1751$ & $4$ \\
Heartbeat & $61$ & $405$ & $2$ \\
MotorImagery & $64$ & $3000$ & $2$ \\
SelfRegulationSCP1 & $6$ & $896$ & $2$ \\
SelfRegulationSCP2 & $7$ & $1152$ & $2$ 
\end{tabular}
\label{tab:UEA_summary}
\end{table}
Tables \ref{tab:UEA_hypopt_recurrent} and \ref{tab:UEA_hypopt_ncde} give an overview of the hyperparameters optimised over during the UEA-MTSCA and PPG-DaLiA experiments for the stacked recurrent models and the NCDE models respectively. The optimisation was performed using a grid search of the validation accuracy for the UEA-MTSCA datasets and the mean squared error for the PPG-DaLiA dataset. All models and experiments used Adam as their optimiser and a batch size of $32$, except the stacked recurrent models on PPG-DaLiA, which used a batch size of $4$ due to memory constraints. NCDEs, NRDEs, and Log-NCDEs use Heun as their differential equation solver with a fixed stepsize of $1 / \max\{500, 1 + (\text{Time series length} / \text{Log-ODE step})\}$, with $\text{Log-ODE step}=1$ for NCDEs. Additionally, Log-NCDEs scale down their initial FCNN parameters by a factor of $1000$ to reduce the starting $\mathrm{Lip}(2)$ norm of the vector field.

\begingroup
\renewcommand{\arraystretch}{1.2}
\begin{table}
\caption{Hyperparameters selected by the optimisation for LRU, S5, S6, and MAMBA on the UEA-MTSCA datasets and PPG-DaLiA dataset. The following abbreviations are used: EigenWorms (EW), EthanolConcentration (EC), Heartbeat (HB), MotorImagery (MI), SelfRegulationSCP1 (SCP1), SelfRegulationSCP2 (SCP2), and PPG-DaLiA (PPG). A \ding{55} denotes that the hyperparameter is not applicable to that model.}
\label{tab:UEA_hypopt_recurrent}
\vspace{0.2cm}
\centering
\begin{tabular}{l|l|c|c|c|c}
\multirow{2}{*}{Hyperparameters} & \multirow{2}{*}{Options} & \multicolumn{4}{c}{Method} \\ \cline{3-6}
 &  & LRU & S5 & S6 & MAMBA \\ \Xhline{2\arrayrulewidth}
Learning Rate & $10^{-3}$ & \makecell{EW, MI, SCP1, \\ SCP2, PPG} & \makecell{HB, MI, SCP1, \\ PPG} & \makecell{EW, HB, MI, \\ SCP2} & EW, EC, PPG \\ \cline{2-6}
 & $10^{-4}$ & HB & EW, SCP2 & SCP1, PPG & HB, SCP2 \\ \cline{2-6}
 & $10^{-5}$ & EC & EC & EC & MI, SCP1 \\ \cline{2-6} \Xhline{2\arrayrulewidth}
Include Time & True & EC, HB, SCP2 & \makecell{EW, EC, SCP2, \\ PPG} & \makecell{EC, HB, MI, \\ PPG} & EW, EC, SCP2 \\ \cline{2-6}
 & False & \makecell{EW, MI, SCP1, \\ PPG} & HB, MI, SCP1 & EW, SCP1, SCP2 & \makecell{HB, MI, SCP1, \\ PPG} \\ \cline{2-6} \Xhline{2\arrayrulewidth}
Hidden Dimension & 16 & MI & MI, SCP2, PPG & \makecell{EW, EC, HB, \\ MI, SCP2} & EW \\ \cline{2-6}
 & 64 & \makecell{EW, EC, SCP1, \\ SCP2} & EW & SCP1, PPG & EC, HB, SCP2 \\ \cline{2-6}
 & 128 & HB, PPG & EC, HB, SCP1 & & MI, SCP1, PPG \\ \cline{2-6} \Xhline{2\arrayrulewidth}
Number of Layers & 2 & HB, SCP1, SCP2 & EW, EC, SCP2 & SCP1, SCP2, PPG & MI, SCP1 \\ \cline{2-6}
 & 4 & EW & HB & EW, EC, HB, MI & EC, HB, PPG \\ \cline{2-6}
 & 6 & EC, MI, PPG & MI, SCP1, PPG & & EW, SCP2 \\ \cline{2-6} \Xhline{2\arrayrulewidth}
State Dimension & 16 & EC, SCP1, SCP2 & \makecell{EW, EC, HB, \\ SCP1} & EC, HB, SCP1 & SCP1 \\ \cline{2-6}
 & 64 & EW & MI, SCP2, PPG & EW, PPG & \makecell{EW, MI, SCP2, \\ PPG} \\ \cline{2-6}
 & 256 & HB, MI, PPG & & MI, SCP2 & EC, HB \\ \cline{2-6} \Xhline{2\arrayrulewidth}
S5 Initialisation Blocks & 2 & \ding{55} & SCP2, PPG & \ding{55} & \ding{55} \\ \cline{2-6}
 & 4 & \ding{55} & HB, MI & \ding{55} & \ding{55} \\ \cline{2-6}
 & 8 & \ding{55} & EW, EC, SCP1 & \ding{55} & \ding{55} \\ \cline{2-6} \Xhline{2\arrayrulewidth}
Convolution Dimension & 2 & \ding{55} & \ding{55} & \ding{55} & EW, HB, SCP2 \\ \cline{2-6}
 & 3 & \ding{55} & \ding{55} & \ding{55} & MI, PPG \\ \cline{2-6}
 & 4 & \ding{55} & \ding{55} & \ding{55} & EC, SCP1 \\ \cline{2-6} \Xhline{2\arrayrulewidth}
Expansion Factor & 1 & \ding{55} & \ding{55} & \ding{55} & EW, MI, SCP1 \\ \cline{2-6}
 & 2 & \ding{55} & \ding{55} & \ding{55} & SCP2, PPG \\ \cline{2-6}
 & 4 & \ding{55} & \ding{55} & \ding{55} & EC, HB \\ \cline{2-6} \Xhline{2\arrayrulewidth}

\end{tabular}
\end{table}
\endgroup

\begingroup
\renewcommand{\arraystretch}{1.2}
\begin{table}
\caption{Hyperparameters selected by the optimisation for NCDE, NRDE, and Log-NCDE on the UEA-MTSCA datasets and PPG-DaLiA dataset. Given the length of each timeseries in the PPG-DaLiA dataset, different choices were considered for the Log-ODE depth and step, which are shown here in red. The following abbreviations are used: EigenWorms (EW), EthanolConcentration (EC), Heartbeat (HB), MotorImagery (MI), SelfRegulationSCP1 (SCP1), SelfRegulationSCP2 (SCP2), and PPG-DaLiA (PPG). A \ding{55} denotes that the hyperparameter is not applicable to that model.}
\label{tab:UEA_hypopt_ncde}
\vspace{0.2cm}
\centering
\begin{tabular}{l|l|c|c|c}
\multirow{2}{*}{Hyperparameters} & \multirow{2}{*}{Options} & \multicolumn{3}{c}{Method} \\ \cline{3-5}
 &  & NCDE & NRDE & Log-NCDE \\ \Xhline{2\arrayrulewidth}
Learning Rate & $10^{-3}$ & \makecell{EW, EC, HB, \\ MI, SCP2, PPG} & \makecell{EW, EC, HB, \\ SCP1, PPG} & \makecell{EW, HB, MI, \\ PPG} \\ \cline{2-5}
 & $10^{-4}$ & SCP1 & MI, SCP2 & EC, SCP1, SCP2 \\ \cline{2-5}
 & $10^{-5}$ & & & \\ \cline{2-5} \Xhline{2\arrayrulewidth}
Include Time & True & \makecell{EW, EC, HB, \\ MI, PPG} & \makecell{EC, SCP1, SCP2, \\ PPG} & \makecell{EW, EC, HB, \\ PPG} \\ \cline{2-5}
 & False & SCP1, SCP2 & EW, HB, MI & MI, SCP1, SCP2 \\ \cline{2-5} \Xhline{2\arrayrulewidth}
Hidden Dimension & 16 & MI, PPG & & HB, MI \\ \cline{2-5}
 & 64 & & \makecell{EW, EC, HB, \\ SCP1} & EC, SCP1 \\ \cline{2-5}
 & 128 & \makecell{EW, EC, HB, \\ SCP1, SCP2} & MI, SCP2, PPG & EW, SCP2, PPG \\ \cline{2-5} \Xhline{2\arrayrulewidth}
Vector Field (Depth, Width) & (2, 32) & & & EW, SCP2 \\ \cline{2-5}
 & (3, 64) & & EW & EC, MI, PPG \\ \cline{2-5}
 & (3, 128) & EW & HB & HB, SCP1 \\ \cline{2-5}
 & (4, 128) & \makecell{EC, HB, MI, \\ SCP1, SCP2, PPG} & \makecell{EC, MI, SCP1, \\ SCP2, PPG} & \\ \cline{2-5} \Xhline{2\arrayrulewidth}
Log-ODE (Depth, Step) & (1, 1) & \ding{55} & EC, MI, SCP1, SCP2 & EC \\ \cline{2-5}
 & (2, 2) & \ding{55} & HB & HB \\ \cline{2-5}
 & (2, 4) & \ding{55} & EW & SCP2 \\ \cline{2-5}
 & (2, 8) & \ding{55} & & \\ \cline{2-5}
 & (2, 12) & \ding{55} & & EW \\ \cline{2-5}
 & (2, 16) & \ding{55} & & MI, SCP1 \\ \cline{2-5}
 & \textcolor{red}{(1, 10)} & \ding{55} & PPG & PPG \\ \cline{2-5}
 & \textcolor{red}{(2, 10)} & \ding{55} & & \\ \cline{2-5}
 & \textcolor{red}{(2, 100)} & \ding{55} & & \\ \cline{2-5}
 & \textcolor{red}{(2, 1000)} & \ding{55} & & \\ \cline{2-5} \Xhline{2\arrayrulewidth}
Regularisation $\lambda$ & $10^{-3}$ & \ding{55} & \ding{55} & EW, MI, SCP2 \\ \cline{2-5}
 & $10^{-6}$ & \ding{55} & \ding{55} & EC, HB \\ \cline{2-5}
 & $0.0$ & \ding{55} & \ding{55} & SCP1, PPG \\ \cline{2-5} \Xhline{2\arrayrulewidth}

\end{tabular}
\end{table}
\endgroup

\subsection{Additional Memory and run time Results}
\label{app:mem_time}

Models are compared on their average GPU memory usage and run time for the UEA-MTSCA datasets. In order to compare the models, $1000$ steps of training were run on an NVIDIA RTX 4090 with each model using the hyperparameters obtained from the hyperparameter optimisation. The specific choices of the hyperparameters are listed in Tables \ref{tab:UEA_hypopt_recurrent} and \ref{tab:UEA_hypopt_ncde}. In addition to the time for $1000$ steps and GPU memory usage, shown in Figures \ref{fig:mem} and \ref{fig:time} respectively, the average number of total training steps taken to produce the results in Table \ref{tab:UEA_results_hypopt} is recorded in Figure \ref{fig:num_steps}. Combining the results for time for $1000$ training steps and total number of training steps gives an approximation of the total run time on the same hardware, and these results are shown in Figure \ref{fig:total_time}.

Although the time per training step is lower for stacked recurrent models than NCDEs, NRDEs, or Log-NCDEs, they also require more training steps to converge. Additionally, NCDEs, NRDEs, and Log-NCDEs require less GPU memory. The largest contributors to the average run time of NCDEs are the datasets with the most observations, EigenWorms and MotorImagery. The positive impact of the Log-ODE method on computational burden is demonstrated empirically by the decrease in run time achieved by NRDEs and Log-NCDEs on EigenWorms when using a depth$-2$ Log-ODE method. When a depth$-1$ Log-ODE method is used, such as NRDEs on MotorImagery, the same decrease is not observed. Section \ref{sec:cost} demonstrated that Log-NCDEs and NRDEs have the same asymptotic computational complexity. However, when using a depth$-2$ Log-ODE approximation and the same stepsize, NRDEs and Log-NCDEs exhibit drastically different run times on Heartbeat, a high-dimensional dataset. This difference is partly explained by the model's having different optimal hyperparameter choices, but even when using identical hyperparameters to the NRDE, Log-NCDE's run time for $1000$ steps of training is $1673$ seconds, whereas NRDE's run time is $9539$ seconds. The remaining difference in run time is due to calculating the JVPs of $f_{\theta}$ using a batched function, as discussed in Section \ref{sec:cost}. If instead the JVPs are calculated using a for-loop, then Log-NCDEs run time increases to $17045$ seconds.

\begin{figure}[h]
\centering
    \begin{subfigure}{0.45\textwidth}
        \includegraphics[width=\linewidth]{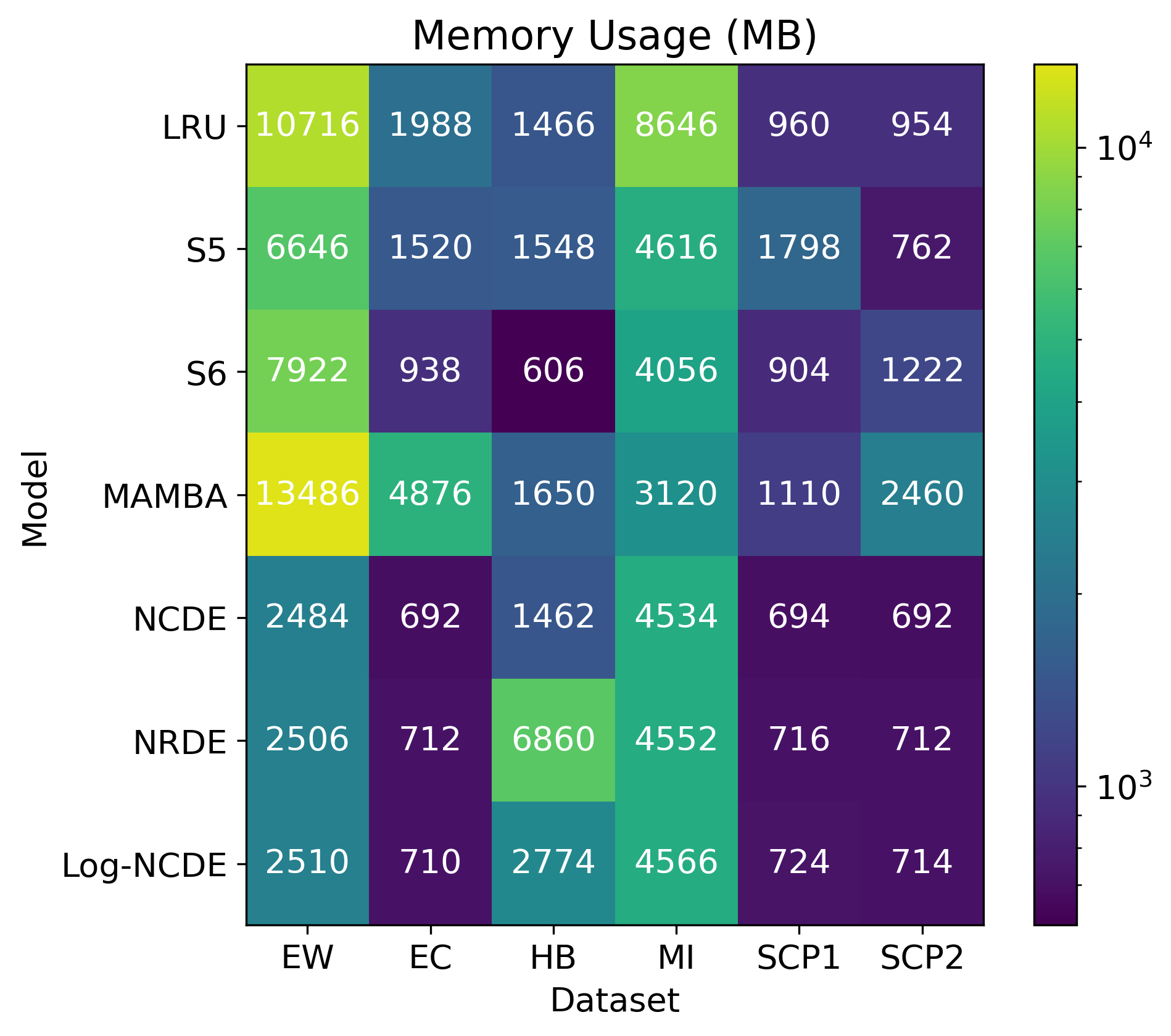}
        \caption{Memory}
        \label{fig:mem}
    \end{subfigure}
    \begin{subfigure}{0.45\textwidth}
        \includegraphics[width=\linewidth]{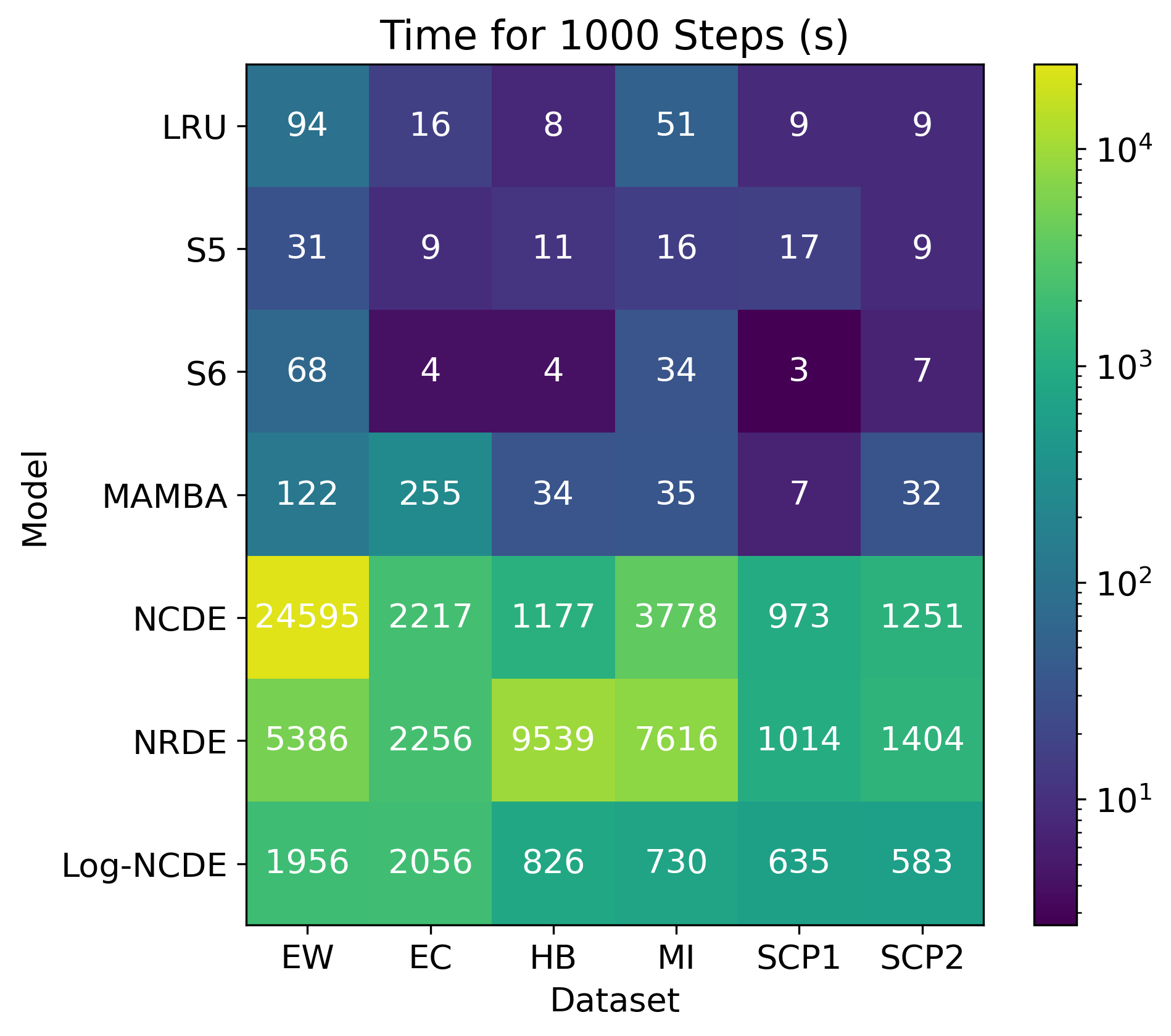}
        \caption{Time}
        \label{fig:time}
    \end{subfigure}

    \medskip
    \begin{subfigure}{0.45\textwidth}
        \includegraphics[width=\linewidth]{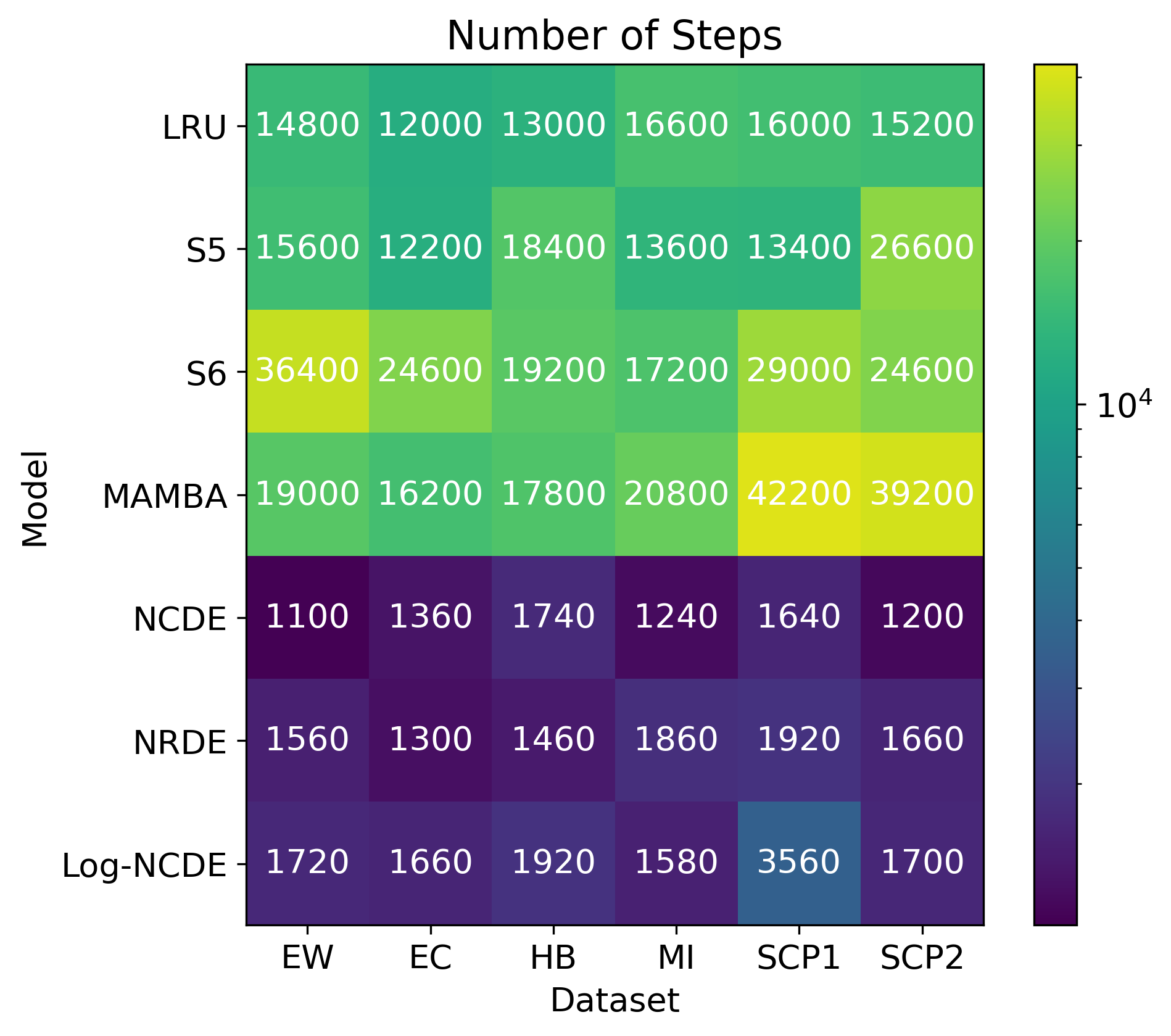}
        \caption{Number of steps}
        \label{fig:num_steps}
    \end{subfigure}
    \begin{subfigure}{0.45\textwidth}
        \includegraphics[width=\linewidth]{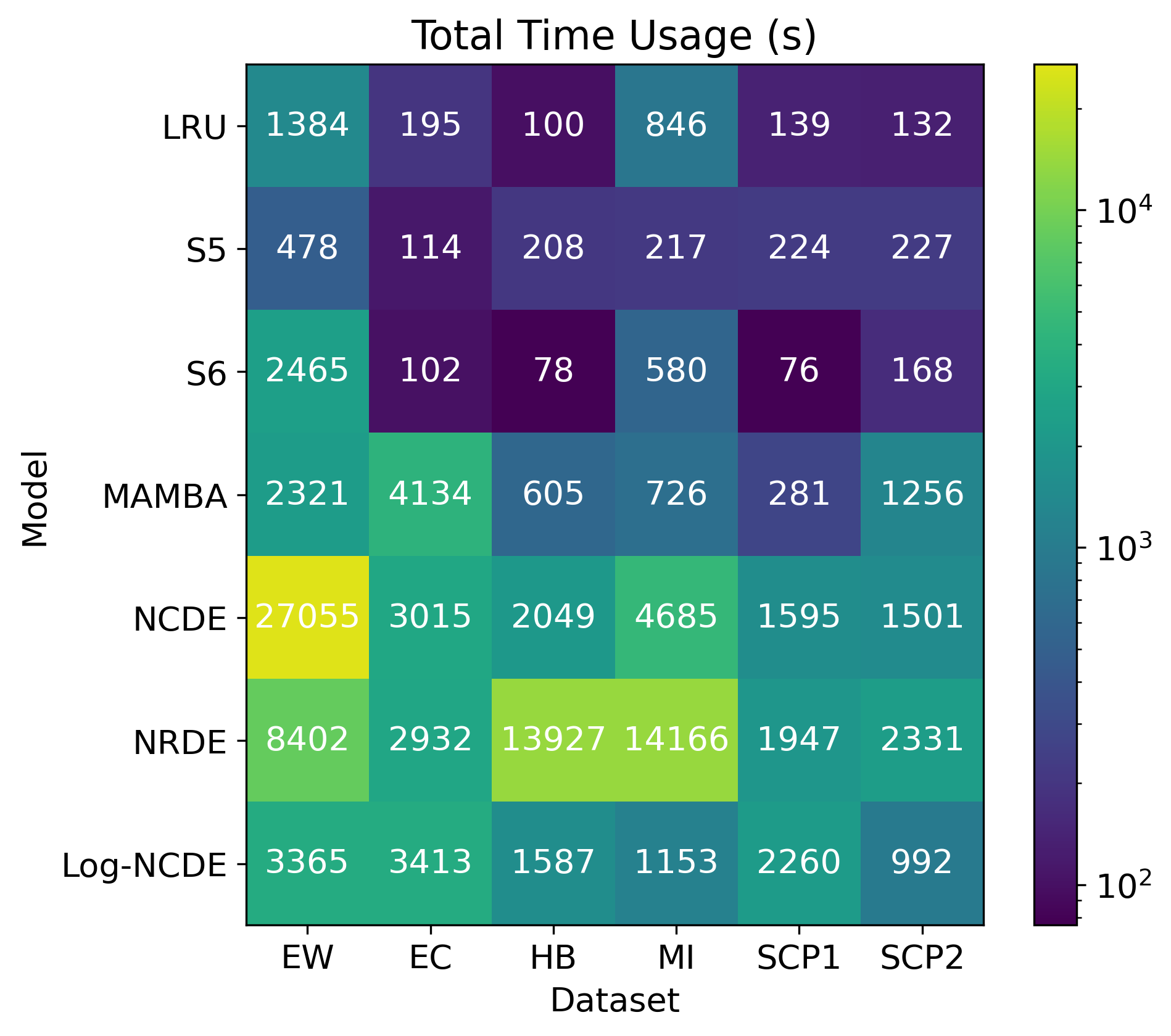}
        \caption{Total time}
        \label{fig:total_time}
    \end{subfigure}
    \caption{Memory, time for $1000$ steps, number of steps, and approximate total time for each model and dataeset from the UEA-MTSCA considered in this paper on an NVIDIA RTX 4090. The following abbreviations are used: EigenWorms (EW), EthanolConcentration (EC), Heartbeat (HB), MotorImagery (MI), SelfRegulationSCP1 (SCP1), and SelfRegulationSCP2 (SCP2).}
    \label{fig:mem_time}
\end{figure}

\end{document}